\documentclass{article}

\usepackage{iclr2027_conference,times}

\usepackage{amsmath,amssymb,amsfonts,amsthm,mathtools}
\usepackage[utf8]{inputenc}
\usepackage[T1]{fontenc}
\usepackage[colorlinks=true,linkcolor=blue!45!black,citecolor=blue!45!black,urlcolor=blue!45!black]{hyperref}
\usepackage{url}
\usepackage{booktabs}
\usepackage{nicefrac}
\usepackage{microtype}
\usepackage{xcolor}
\usepackage{graphicx}
\usepackage{tabularx}
\usepackage{array}
\usepackage{multirow}
\usepackage{placeins}
\usepackage[capitalize,noabbrev]{cleveref}
\usepackage{tcolorbox}

\theoremstyle{plain}
\newtheorem{theorem}{Theorem}[section]
\newtheorem{proposition}[theorem]{Proposition}
\newtheorem{lemma}[theorem]{Lemma}
\newtheorem{corollary}[theorem]{Corollary}
\theoremstyle{definition}
\newtheorem{definition}[theorem]{Definition}

\theoremstyle{remark}
\newtheorem{remark}[theorem]{Remark}

\iclrfinalcopy

\title{From Discovery to Decision: \\Finite-Budget Recoverability in LLM Voting}

\author{Shaoang Li \\
Stony Brook University \\
\texttt{shaoang.li@stonybrook.edu}
\And
Jian Li \\
Stony Brook University \\
\texttt{jian.li.3@stonybrook.edu}}

\begin{document}

\maketitle
\lhead{Preprint}

\begin{abstract}

Voting over multiple LLM responses is a common primitive in test-time scaling and ensemble inference. Collecting more responses can expand the candidate pool and increase the chance that a correct answer is discovered. Under a fixed call budget, a discovered answer still needs to accumulate enough support within the remaining calls to become the final plurality winner, creating a discovery-to-decision gap. In this work, we characterize this gap through the realized vote state and remaining call budget. We derive a sharp recoverability threshold and show that, as sampling proceeds, the observed candidate set can only expand while the set of reachable endpoint winners can only contract, inducing a candidate-level conversion window. 
Under a specified iid response law, the same state yields exact finite-horizon endpoint probabilities. 
We further show that merging wrong-answer identities preserves single-call correctness and cannot improve plurality accuracy, and that the effect of redistributing wrong-answer probability depends on the realized vote state. Singleton reachability yields a gold-free exact locking certificate. For a known answer universe, its first trigger is the earliest prefix at which all admissible continuations yield the same fixed-budget output. Empirically, most discovered-but-unselected correct answers lose reachability only after discovery. In a controlled Word16 study, input permutation improves raw-plurality accuracy by 21.1 points with essentially unchanged single-call correctness. Exact locking saves 28--30\% of calls at a 16-call budget while preserving every fixed-budget output.

\end{abstract}

\section{Introduction}
\label{sec:intro}

Repeated sampling is a standard way to spend additional inference-time compute with language models.
Voting aggregates sampled responses into a single prediction and is widely used in self-consistency, test-time scaling and multi-agent inference \citep{wang2023selfconsistency,brown2024largelanguagemonkeys}.
Collecting more responses can expand the candidate pool and increase the chance of discovering a correct answer, but discovery does not guarantee that the answer will be selected \citep{stroebl2024flaws,bay2026modalceiling}.
Under a fixed call budget, a newly discovered answer must accumulate enough support within the remaining calls to become the final plurality winner. 

For example, if a correct answer first appears with one vote while a wrong answer already has five and only three calls remain, even an all-correct suffix cannot make the correct answer win. 
An early discovery can also fail when later votes strengthen a rival, while a late discovery can still recover if its deficit is small enough. A correct answer may also become the current plurality winner and subsequently lose at the endpoint. Thus, the opportunity to win depends on the \emph{current vote state}: the accumulated vote counts together with the remaining call budget. We study this \emph{discovery-to-decision gap}: when a discovered answer can still become the fixed-budget winner and when that opportunity has disappeared. These possibilities motivate distinguishing five events along a trajectory: \emph{discovery}, \emph{recoverability}, \emph{current selection}, \emph{endpoint selection}, and \emph{decision locking}. \Cref{fig:trajectories} illustrates why these events need not coincide. This state view leads to our central question:
\vspace{-0.09in}
\begin{tcolorbox}[colback=white!5!white,colframe=white!75!white]
\emph{Given the realized vote state and remaining budget, what can additional inference still change: which answers can still win, how likely they are to do so,
and when the fixed-budget output is already determined?}
\end{tcolorbox}
\vspace{-0.1in}

We answer this question with a \textbf{finite-budget recoverability} framework for plurality voting with retained unit votes and fixed tie breaking. Given the realized vote state and remaining budget, each candidate has a sharp minimum-support threshold for becoming the endpoint winner. These thresholds define a reachable-winner set that can only contract as the budget is spent, while the observed candidate set can only expand. The prefixes at which a candidate is both observed and reachable form its possibly empty, contiguous \emph{conversion window}. Under a specified iid response law, the same state also determines exact finite-horizon endpoint probabilities.

\begin{figure}[t]
\centering
\includegraphics[width=\linewidth]{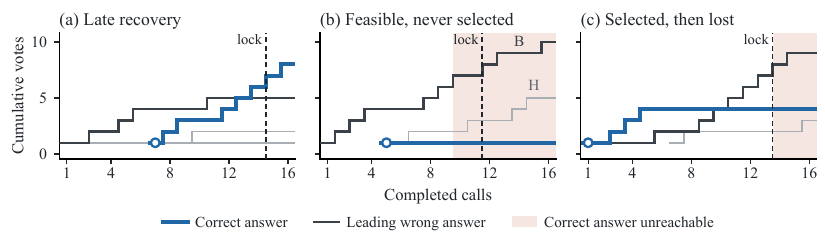}
\par\vspace{-0.2em}
\vspace{-0.1in}
\caption{Three 16-call trajectories (Qwen3.5-2B on forced-choice MMLU-Pro, chosen retrospectively with gold labels). Circles mark discovery, thin gray lines show other wrong answers, and dashed lines mark decision locking. In (b), the correct answer is unreachable from call 10, but H can still overtake B, so the output locks only at call 12.}
\label{fig:trajectories}
\vspace{-0.2in}
\end{figure}

Empirically, these distinctions matter. Across three response protocols, most correct answers that are discovered yet absent from the endpoint decision remain reachable when first discovered and lose reachability only later. We evaluate the probabilistic layer prospectively using plug-in response laws estimated separately from the suffixes they predict. The resulting forecasts distinguish which unresolved states are more likely to recover, while short response-law estimates limit calibration near the extremes.

Beyond reachability, endpoint success also depends on which wrong-answer identities repeatedly accumulate support. We prove that merging wrong-answer identities weakly decreases plurality accuracy while preserving single-call correctness, and derive a state-dependent pivotality formula for how transfers of wrong-answer probability mass change endpoint success. Prospective effect forecasts and oracle recoding support this state dependence. These results motivate independently permuting input representations across calls to diversify recurring errors. On Word16, permutation raises Qwen3.5-4B raw-plurality accuracy by 21.1 percentage points while leaving single-call correctness essentially unchanged, demonstrating that substantial aggregation gains can arise from changes in the structure of repeated errors. The gain attenuates under selectors that already remove or fuse repeated wrong answers, and this generator--selector interaction replicates on fresh inputs.

Finally, singleton reachability yields a gold-free certificate that the fixed-budget plurality output is already determined for every admissible continuation. For a known answer universe, its first trigger is the earliest universally valid exact stopping time among prefix-only policies that must preserve the fixed-budget output. The certificate remains exact with budget-consuming abstentions and adaptive proposal policies, and a conservative extension covers unseen competitors. We further derive a lock-time law: under iid calls, the asymptotic saving fraction is governed by the modal margin, with a central limit theorem describing its fluctuations. This exact-locking structure is closely related to classical deterministic curtailment in sequential analysis \citep{alling1963early,bechhofer1959multinomial}. The predicted savings match large-budget public data. At a 16-call budget, replay saves 28--30\% of calls while preserving every fixed-budget output, and physical execution confirms request-time savings.

Our main contributions are summarized as follows:

\begin{itemize}
    \item \textbf{Finite-budget reachability and conversion windows.} We formulate repeated LLM sampling through the evolving vote state and derive a sharp threshold characterizing which candidates can still become the endpoint winner. The observed candidate set expands monotonically while the reachable-winner set contracts, inducing a possibly empty, contiguous conversion window for each candidate. Under a specified iid response law, the same state yields exact finite-horizon endpoint probabilities.

    \item \textbf{Error competition and representation diversity.}
    We prove that merging wrong-answer identities preserves single-call correctness and cannot improve plurality accuracy or lower recovery thresholds, and derive a state-dependent pivotality formula for probability transfers across wrong identities. A pre-specified Word16 study shows that representation diversification raises raw-plurality accuracy with essentially unchanged single-call correctness, and that the gain depends on the endpoint selector.

    \item \textbf{Exact decision locking.} We show that singleton reachability yields a gold-free certificate that exactly preserves the fixed-budget plurality output for every admissible continuation, including dependent responses and abstentions. For a known answer universe, its first trigger is the earliest universally valid exact stopping time among prefix-only policies that must preserve the fixed-budget output. We characterize its savings with an asymptotic limit and a central limit theorem. We validate the predicted savings in replay and on large-budget public data, and we measure request-time savings in physical execution, with every fixed-budget output preserved.
\end{itemize}

\begin{samepage}
\section{A State-Based Theory of Finite-Budget Voting}
\label{sec:structure}\label{sec:laws}

\paragraph{Five distinct events.}
We distinguish five events along a finite-budget voting trajectory: \emph{discovery} (the correct answer has appeared), \emph{recoverability} (some continuation within the remaining call budget makes a discovered answer the endpoint winner), \emph{current selection} (a discovered answer is the plurality winner at the observed prefix), \emph{endpoint selection} (a discovered answer is the winner after the full budget is spent), and \emph{decision locking} (every admissible continuation produces the same endpoint answer, which may be wrong).

\end{samepage}

\subsection{Voting States and Remaining Budget}

Fix an item $x$ with correct (gold) answer $y^\star\in\mathcal{U}_x$, where $\mathcal{U}_x$ is a finite answer universe, and let $B$ be the total call budget. A fixed interpretation rule maps each response to an answer $y\in\mathcal{U}_x$ or to an abstention $\bot$. Let $\kappa_x$ be an injective tie priority, with larger values winning ties. Each call consumes one unit of budget and appends at most one unit vote without changing earlier votes.

After $N$ calls, let $r=B-N$ be the remaining budget and let $v_N=(v_N(y))_{y\in\mathcal{U}_x}$ be the accumulated count state. The observed candidate set is $\mathcal{C}_N=\{y:v_N(y)>0\}$. Plurality maximizes $(v_N(y),\kappa_x(y))$ over candidates and returns $\bot$ when the pool is empty. All deterministic results below allow dependent or adaptive responses, and probability calculations additionally specify an iid law for future calls. We call any length-$m$ future sequence in $(\mathcal{U}_x\cup\{\bot\})^m$ an \emph{admissible continuation} when it obeys this fixed update rule: each call appends one unit vote to its interpreted answer, or appends no vote under $\bot$, while earlier votes and the tie priority remain unchanged.

Let $A(v)$ indicate that plurality selects gold from state $v$. Coverage after $N$ calls is the event $y^\star\in\mathcal{C}_N$, and accuracy is $\mathbb{E}[A(v_N)]$.

\subsection{Reachability and the Conversion Window}

Reachability gives the first answer to our central question: \textit{which endpoint decisions remain attainable from the realized prefix?} Consider a count state $v$ with $r$ calls remaining. Suppose gold has one vote and the leading wrong answer has five. Gold needs four further votes if it wins their count tie and five if it loses the tie, assuming the rival receives no additional votes. In general, write $v_\star=v(y^\star)$ and $\delta_a=\mathbf{1}\{\kappa_x(a)>\kappa_x(y^\star)\}$, and define
\begin{equation}
 L(v)=\max\left\{0,\ 1-v_\star,\ \max_{a\ne y^\star}\bigl(v(a)-v_\star+\delta_a\bigr)\right\}.
 \label{eq:finite-budget-required}
\end{equation}
The presence term $1-v_\star$ requires at least one gold vote when gold has not yet appeared.
Replacing gold by an arbitrary answer $y\in\mathcal{U}_x$ in \cref{eq:finite-budget-required} gives $L_y(v)$, and we write
\begin{equation}
 \mathcal{R}_r(v)=\{y\in\mathcal{U}_x:L_y(v)\leq r\}.
 \label{eq:reachable-winners}
\end{equation}
Reachability is combinatorial: it ranges over all answer continuations allowed by the remaining call budget, including continuations that a particular response law may assign zero probability. It concerns answer candidates in $\mathcal{U}_x$ and excludes $\bot$, which can arise only from an empty state. At a fixed total budget,
\begin{equation}
 \mathcal{C}_N\subseteq\mathcal{C}_{N+1},\qquad
 \mathcal{R}_{B-N-1}(v_{N+1})\subseteq\mathcal{R}_{B-N}(v_N).
 \label{eq:shrinking-window}
\end{equation}

\begin{proposition}[Reachable winners]
\label[proposition]{prop:reachability}
(i) From count state $v$, no continuation of fewer than $L_y(v)$ calls makes $y$ the plurality winner, even at an intermediate prefix, and for every $m\geq L_y(v)$ the length-$m$ continuation that votes only for $y$ ends with $y$ as the winner. Hence $\mathcal{R}_r(v)$ is exactly the set of answers that some continuation of the remaining $r$ calls makes the endpoint winner. In particular, gold is recoverable if and only if $r\geq L(v)$. (ii) At a fixed total budget $B$, every admissible update satisfies \cref{eq:shrinking-window}: observed candidate sets expand, reachable-winner sets contract, and an answer that leaves the reachable set cannot return within the same budget.
\end{proposition}
Both parts specialize possible-winner reasoning \citep{lang2020possible} to plurality with $r$ added votes: a win needs $L_y(v)$ further votes for $y$, which the all-$y$ continuation supplies, and every continuation after an update was available before it (proofs, including abstentions: \cref{app:finite-budget}).

\begin{definition}[Candidate conversion window]
\label{def:conversion-window}
For any candidate $y$ first observed at $D_y=\min\{N:v_N(y)>0\}$, its conversion window is $\mathcal{W}_y=\{N\in\{D_y,\ldots,B\}:y\in\mathcal{R}_{B-N}(v_N)\}$. By \cref{prop:reachability}(ii), a nonempty window runs from $D_y$ to the last prefix at which $y$ is reachable. It is empty when $y$ is already infeasible at discovery and extends through $B$ when $y$ remains reachable at the endpoint.
\end{definition}

After discovery, a candidate can still become the fixed-budget winner only while it remains reachable. When the reachable set becomes a singleton in a nonempty state, the output is locked for every continuation (\cref{app:decision-locking,sec:stopping-replay}).

\subsection{Endpoint Recovery under a Specified Response Law}

Reachability determines which endpoint outcomes remain possible. Their probabilities additionally depend on the future response process. Let future calls be iid with unconditional probabilities $p_y$ for answers and $p_\bot$ for abstention, with $p_\bot+\sum_y p_y=1$. Write $p$ for this response law, and let $Q_r(v;p)$ denote the probability that gold is selected after exactly $r$ additional calls from state $v$. With $e_y$ denoting one additional vote for $y$,
\begin{equation}
 Q_0(v;p)=A(v),\qquad
 Q_{r+1}(v;p)=p_\bot Q_r(v;p)+\sum_{y\in\mathcal{U}_x}p_y Q_r(v+e_y;p).
 \label{eq:abstention-recursion}
\end{equation}
The recursion accounts jointly for future gold votes, competing votes, and abstentions. If $p_{y^\star}>0$, then $Q_r(v;p)>0$ exactly when $r\geq L(v)$, so the probabilistic and combinatorial views share one boundary (binomial tail bound, exact multinomial formula, abstention mixture and first-passage contrast: \cref{app:sharp-threshold,app:endpoint-laws}).

\paragraph{Unknown response law.}
When the future law is unknown, a plug-in estimate $\widehat p$ gives the forecast $Q_r(v;\widehat p)$, whose accuracy also depends on estimating the response law.  \cref{sec:conversion,sec:error-concentration} evaluate such forecasts prospectively.

\section{Empirical Evidence for State-Dependent Recovery}
\label{sec:conversion}\label{sec:experiments-m}

We ask whether these state distinctions matter empirically: whether discovered answers stay reachable and convert, and how much the realized state predicts.

\subsection{Protocols}

Three response protocols are analyzed separately (\cref{tab:study-map}). Forced-choice MMLU-Pro \citep{wang2024mmlupro} (200 items and 2,000 trajectories from Qwen3.5-0.8B/2B/4B, Gemma-4-E2B and Ministral-3-3B) turns each generated rationale into one vote by scoring the answer options. GSM-Symbolic \citep{mirzadeh2024gsmsymbolic} (250 items from 100 template families, with the same five checkpoints and 6,250 trajectories) and native generation (Qwen3.5-2B/4B on 1,200 MMLU-Pro items and 200 items of Word16, a controlled word-sorting task, with 2,800 trajectories) parse a generated final answer under native chat templates, and invalid responses consume calls without adding votes. Intervals are pointwise 95\% and clustered by item or GSM template family. State classifications use gold labels only offline. \Cref{app:experimental-overview} gives the pipelines and the scoring sensitivity.

\subsection{Failure and Recovery across Protocols}
\label{sec:native-validation}

Coverage exceeds endpoint accuracy in every cohort (\cref{fig:gap}a). Each gap counts correct answers present in the pool but absent from the final decision. In the $B=16$ cohorts of \cref{tab:unified-failures}, between 2\% and 46\% of the trajectories that sample the correct answer end with a wrong plurality output. In 10,000 public samples per problem from Llama-3-8B and 70B \citep{brown2024largelanguagemonkeys}, the gap is small for 70B on GSM8K \citep{cobbe2021gsm8k} (2.4 points at $B=16$) and widens with the budget on MATH \citep{hendrycks2021math}. From $B=16$ to $1{,}024$, 70B coverage on MATH rises from 80.0\% to 95.2\% while plurality accuracy stays near 60\%, and the 8B gap grows from 27.6 to 48.4 points (\cref{app:scope-scale}).

We classify these failures as \emph{infeasible at discovery} ($L(v_t)>B-t$, where $t$ is the call that adds the first gold vote), \emph{feasible but never selected}, or \emph{selected then lost}. These mutually exclusive, gold-aware classes separate inability to recover from failure to acquire or retain the lead, and \cref{app:experimental-overview} gives the common definition and discovery-time convention.

\begin{figure}[t]
\centering
\includegraphics[width=\linewidth]{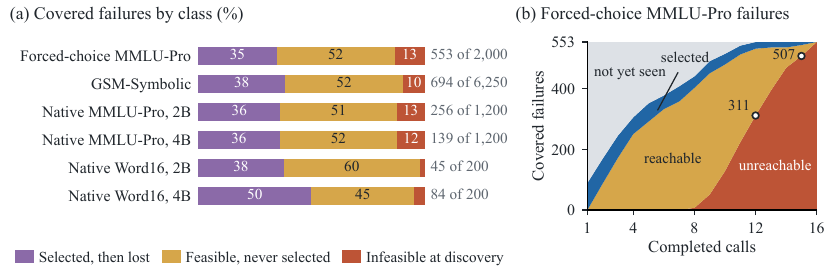}
\vspace{-0.2in}
\caption{Where covered correct answers are lost. (a) Covered failures (the correct answer is sampled but not selected at the endpoint) by class, as a percentage of each cohort's covered failures. The labels at right give covered failures out of all trajectories (counts in \cref{tab:unified-failures}). (b) The 553 forced-choice MMLU-Pro covered failures by the correct answer's state after each call.}
\label{fig:gap}\label{fig:failure-composition}
\vspace{-0.2in}
\end{figure}

Impossibility usually arrives after discovery. In every $B=16$ cohort, 81--98\% of covered failures were still recoverable when the correct answer first appeared (\cref{tab:unified-failures}). In forced-choice MMLU-Pro, only 74 of 553 failures are already infeasible at discovery, but 311 are infeasible by call 12 and 507 before the final call (\cref{fig:gap}b, with the complete conversion-window timing in \cref{tab:historical-window-audit}).

Late discovery, after call 8, can still convert. GSM converts 41/219 late discoveries (18.7\%, $[13.5,23.9]$). Native MMLU-Pro converts 13/82 and 3/33 for 2B and 4B, and Word16 converts 8/21 and 3/21 (post-discovery feasibility convention, \cref{app:gsm-sensitivities}). Forced-choice MMLU-Pro converts none of its 120 late discoveries, although 46 remain feasible, so the signal carried by discovery time depends on the protocol, while the reachability boundary is exact in every protocol.

\subsection{State-Conditioned Forecasting}

We use prospective forecasting to measure how much information the realized vote state carries about subsequent endpoint selection. GSM-Symbolic uses a separate 16-call sequence to estimate each item's response law. Four further sequences supply call-8 states whose call-16 forecasts were fixed before suffix generation (\cref{app:gsm-forecasting}).

Among 490 covered-but-unselected call-8 states across two evaluation waves, the plug-in forecast has AUC $0.825$ $[0.77,0.88]$. Its mean observed-minus-forecast error is $-0.4$ percentage points. The error concentrates in the extreme ranges, where short law estimates matter most. The lowest and highest forecast ranges have errors of $+5.5$ and $-11.9$ points (\cref{fig:forecast-calibration}a). In 55 lowest-range states the correct answer never appears in the 16-call law sequence, so the forecast is exactly zero (\cref{app:gsm-sensitivities}).

\begin{figure}[t]
\centering
\includegraphics[width=\linewidth]{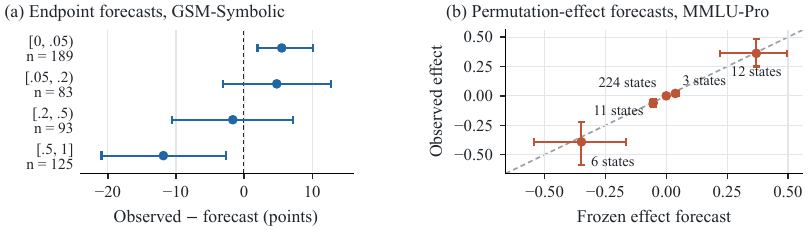}
\vspace{-0.25in}
\caption{Forecasting from the realized state. (a) GSM endpoint calibration in the covered-but-unselected call-8 stratum ($n=490$, laws from 16 calls): observed-minus-forecast error by pre-specified forecast range, with family-bootstrap 95\% intervals. The pooled analysis includes an adaptively added second wave (\cref{tab:gsm-calibration-waves}). (b) Prospective Permuted-minus-Base effect forecasts against observed effects on 128 MMLU-Pro inputs (laws from 256 calls per arm), by pre-specified bin, with item-bootstrap 95\% intervals (\cref{sec:error-concentration}).}
\label{fig:forecast-calibration}
\vspace{-0.15in}
\end{figure}

State information improves first-wave Brier scores over discovery-bin, law-only and persistence baselines, and the plug-in forecast, which fits no parameter to outcomes, matches a gold-aware logistic state comparator trained on MMLU-Pro outcomes within sampling error (\cref{tab:gsm-brier}). Short-sample law estimation and the iid assumption both limit calibration (\cref{app:gsm-sensitivities}). Retrospective replays with longer law samples support a role for the first. With 2,000-sample laws the public-data discrepancies in this stratum stay below 0.4 points, and 16-sample laws miss the highest range by 12 to 28 points on three of four model--task pairs (\cref{app:scope-forecast}). Dependence among calls remains unresolved.

\section{Error Competition Shapes Recoverability}
\label{sec:error-concentration}

Feasibility leaves open which reachable candidate will ultimately win. The key additional object is error competition: repeated probability mass on the same wrong identities creates stronger rivals and changes both recovery thresholds and endpoint success.

\subsection{Controlled Error-Identity Merging}

\begin{proposition}[Merging wrong identities cannot improve plurality]
\label[proposition]{prop:wrong-merging}
Fix a realized response sequence, retaining its gold and abstention positions. Partition wrong identities into blocks and replace every answer in a block by one wrong identity whose tie priority against gold is at least as strong as every member's. At every budget, a gold win after merging implies a gold win before merging. Thus the coupling preserves coverage and single-call correctness and weakly decreases plurality accuracy under any distribution of sequences.
\end{proposition}

Each merged block has at least each member's count and tie advantage, while gold is unchanged, so every adjusted deficit weakly increases: along the coupling, the merged state satisfies $L(v^{\mathrm{merge}})\geq L(v)$ at every prefix, and a state in which gold is infeasible stays infeasible at the same remaining budget (proof and splitting examples: \cref{app:concentration-design}). The effect can be large: with per-call probabilities 0.3 for gold and 0.1 for abstention, and with gold losing count ties, splitting one wrong answer into 32 equally likely ones raises 16-call plurality accuracy from 7.4\% to 88.7\% at unchanged coverage.

A complementary selector-side result appears in \cref{prop:necessary-filter}: an input-defined necessary-condition filter preserves every correct plurality win and weakly lowers the recovery threshold of every retained correct candidate. Word16's exact multiset check instantiates this control and lets us separate generator-side diversity from selector strength in the experiment below. On 384 fresh Word16 inputs, filtering the same Permuted pools raises plurality accuracy from 74.2\% to 92.2\%, correcting 69 inputs and, as guaranteed, damaging none (Holm-adjusted exact $p<10^{-20}$). It removes 85\% of wrong votes yet \emph{raises} their collision conditional on error by $0.24$ $[0.19,0.30]$, so conditional concentration can move against accuracy. Exact collision counterexamples and the filtering proof appear in \cref{app:concentration-design}.

\paragraph{State-dependent probability transfers.} More generally, the effect of redistributing wrong-answer probability depends on the realized state. Keep gold and abstention probabilities fixed and move mass $\epsilon$ from wrong answer $a$ to wrong answer $b$. Call a rival pivotal when, after $r-1$ future calls, gold wins and one additional vote for that rival removes the win.

\begin{proposition}[Pivotal transfer]
\label[proposition]{prop:pivotal-transfer}
For iid future calls with law $p$, distinct wrong answers $a,b$, $r\geq1$ and every $\epsilon$ in the legal range, with both probabilities under the transferred law,
\begin{equation}
\frac{d}{d\epsilon}Q_r(v;p+\epsilon(e_b-e_a))=r\{\Pr(a\text{ pivotal})-\Pr(b\text{ pivotal})\}.
\label{eq:wrong-transfer-derivative}
\end{equation}
\end{proposition}
Moving mass away from the more frequently pivotal rival therefore helps (proof: \cref{app:state-competition}), and the sign depends on the state: with one call left, gold losing ties and probabilities $(0.3,0.6,0.1)$ on gold, $a$ and $b$, moving $0.1$ from $a$ to $b$ lowers the wrong-answer collision yet takes gold's endpoint probability from $0.9$ to $0.8$ at counts $(2,0,1)$ and from $0.4$ to $0.5$ at $(2,1,0)$. No statistic of the response law alone, collision included, can therefore order endpoint accuracy (conditional sign thresholds: \cref{app:state-competition}).

\paragraph{A prospective test.}
For Qwen3.5-4B on 128 separate MMLU-Pro inputs, we estimated each arm's response law from 256 reference calls and froze, for every realized 8- and 12-call Base prefix, the exact forecast of the Permuted-minus-Base endpoint effect at $B=16$ before generating 32 independent continuations per arm and prefix. The state-based forecast beats a control that spreads wrong mass uniformly and one that keeps the whole wrong-mass spectrum, hence its collision, but shuffles which wrong answers carry it (squared-error gains $0.0058$ $[0.0010,0.0134]$ and $0.0030$ $[0.0004,0.0068]$, with Bonferroni 97.5\% item-bootstrap intervals). Half of the states are decided by their prefix. Where the forecast effect exceeded $0.02$ in magnitude, its sign was right in 29 of 32 states and the other three did not move (\cref{fig:forecast-calibration}b and \cref{app:state-forecast}). An oracle recoding of the Base continuations tests \cref{prop:pivotal-transfer} directly: relabeling, with probability $1/2$, each response of the most-voted wrong answer as the least-voted one was predicted to raise the endpoint rate by $1.03$ points and raised it by $1.06$ $[0.46,1.82]$, and the reverse gave $-0.12$ and $-0.07$ (\cref{tab:state-forecast-oracle}). Which answers carry the wrong mass predicts an intervention's effect beyond how concentrated that mass is.

\subsection{Representation Diversity and Selector Interaction}

Word16 is a controlled mechanism task: sort 16 words while preserving multiplicities. An exact programmatic solver exists. We use it, together with MMLU-Pro, for a pre-specified representation-diversification study with 64 development and 256 test items per task. Base and Permuted policies sample 16 responses at temperature $0.7$. The latter independently shuffles inputs per call under a fixed rule. The broader pre-specified family also contains high-temperature and response-dependent policies, all charged the same 16-call budget (\cref{app:adaptive-sampling}).

On the 256 paired 4B Word16 test items, permutation raises raw-plurality accuracy from $57.42\%$ to $78.52\%$, a gain of $21.09$ percentage points (bootstrap 95\% CI $[14.06,28.12]$, Holm-adjusted exact $p=5.54\times10^{-8}$). Single-call correctness changes by $-0.22$ points $[-2.71,2.17]$, while mean all-item wrong collision falls from $0.0908$ to $0.0373$. The other two pre-specified primary comparisons are inconclusive after Holm correction, so no direction is established for either. One is 4B Base versus 2B Base. There, 4B has 11.96 points higher single-call correctness and 8.20 points lower plurality accuracy ($p=0.105$). The other compares the response-dependent Probe policy with the development-selected High temperature comparator (\cref{tab:primary-test-details,app:adaptive-sampling}).

\begin{figure}[!htbp]
\centering
\includegraphics[width=\linewidth]{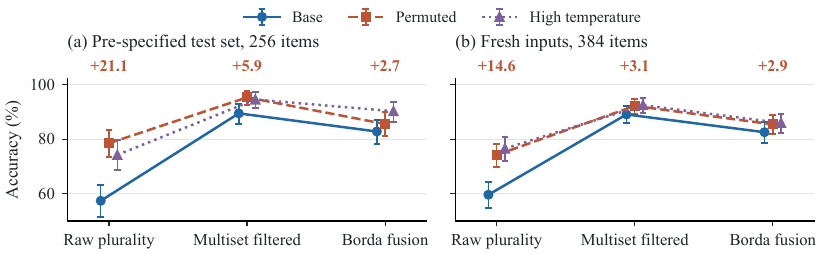}
\vspace{-0.25in}
\caption{Generator--selector interaction for Qwen3.5-4B on Word16, 16 responses per policy: (a) 256 pre-specified test items, (b) 384 fresh inputs. Numbers give the Permuted-minus-Base gain (points). Bars are 95\% item-bootstrap intervals. Paired intervals and cells: \cref{app:diversification-controls} for (a), \cref{app:fresh-mechanism} for (b).}
\vspace{-0.1in}
\label{fig:selector-interaction}
\end{figure}

Representation diversity and selector strength interact. In retrospective selector controls, the permutation gain falls from 21.1 points under raw plurality to 5.9 under multiset filtering and 2.7 under Borda fusion (\cref{fig:selector-interaction}a), consistent with error competition acting through identity counts: a selector that removes or fuses wrong answers leaves less for generator-side diversity to fix. Multiset filtering uses no gold or extra model calls and implements the necessary-condition guarantee. Borda fuses word rankings and may synthesize a full answer absent from the pool, so its endpoint differs from retained-answer plurality. The interaction replicates on the 384 fresh inputs, whose selectors and contrasts were fixed before the outcomes were opened (\cref{fig:selector-interaction}b, with interactions $-11.5$ $[-18.0,-5.0]$ and $-11.7$ $[-19.5,-3.9]$ points for filtering and Borda, 98.3\% intervals). There, high temperature matches or exceeds permutation under every selector but also raises single-call correctness (33.4\% to 36.5\%), whereas permutation leaves it essentially unchanged ($-1.2$ $[-3.3,0.8]$ points) and so isolates the change in error structure. On MMLU-Pro the permutation effect is $+6.25$ $[3.12,9.77]$ points on the study's input-filtered population, $+2.6$ $[-0.3,5.5]$ on a fresh 384-input cohort and $+0.26$ $[-1.56,2.08]$ on a pre-specified 768-input hold-out. The last estimate meets a descriptive, unadjusted $\pm2$-point equivalence criterion at temperature 0.7, and the effect at 1.0 and the temperature interaction are inconclusive. Neither fresh cohort detects an accuracy benefit, although permutation still lowers conditional wrong collision (\cref{app:diversification-secondary,app:fresh-mechanism,app:temp-perm}).

\section{Exact Decision Locking}
\label{sec:stopping-replay}

The contraction of the reachable-winner set gives an executable, gold-free boundary for additional inference: once the current leader is the only reachable endpoint winner, every admissible continuation produces the same fixed-budget plurality output.

\begin{corollary}[Gold-free decision locking]
\label[corollary]{cor:decision-locking}
In a nonempty state with winner $w$, retained unit votes and fixed tie priority,
\begin{equation}
 \mathcal{R}_r(v)=\{w\}
 \quad\Longleftrightarrow\quad
 v(w)-v(a)\geq r+\mathbf{1}\{\kappa_x(a)>\kappa_x(w)\}\quad\forall a\ne w.
 \label{eq:decision-locking}
\end{equation}
Stopping at the first certified prefix, with $r=B-N$, preserves the budget-$B$ output for every continuation, including dependent responses and abstentions.
\end{corollary}

Even assigning every remaining vote to one rival cannot defeat this winner. Gold infeasibility and locking are distinct, since several wrong answers can remain reachable after gold is not (\cref{app:decision-locking} gives an example). The certificate preserves errors as well as correct decisions.

With $w_N$ the current winner whenever $\mathcal{C}_N\neq\varnothing$, the first certified prefix is
\begin{equation}
\tau_B=\min\Bigl(\{N\in\{1,\ldots,B\}:\mathcal{C}_N\neq\varnothing,\ \mathcal{R}_{B-N}(v_N)=\{w_N\}\}\cup\{B\}\Bigr).
\label{eq:earliest-exact-stop}
\end{equation}

\begin{remark}[No earlier exact stop]
\label[remark]{rem:earliest-exact-stopping}
For a known finite answer universe, $\tau_B$ is by construction the first prefix at which every admissible continuation yields the same budget-$B$ output, so no prefix-only rule that must preserve that output can stop earlier (\cref{app:decision-locking}). Stopping sooner needs assumptions on future responses.
\end{remark}

For open answer sets, we additionally require $v(w)>r$ alongside the observed-rival checks to protect against unseen competitors. This conservative condition triggers at $\tau_B$ or one call later. The known-option MMLU-Pro setting attains $\tau_B$ exactly. No independence or response-law estimate is required for safety (proofs and scope: \cref{app:decision-locking}). 
The certificate determines when exact stopping is valid, and Theorem 5.3 below characterizes how much it saves.

Replaying stored trajectories saves 28.35\% of calls on 2,000 MMLU-Pro trajectories and 30.26\% on 5,000 GSM forecast trajectories, preserving every fixed-16 answer, including errors (\cref{tab:decision-locking}). The rule has no tunable parameter, so replaying it on stored trajectories involves no analysis choices. On the 2,000 MMLU-Pro trajectories it also saves more calls than every other evaluated rule that changed no output (\cref{app:curtailment}).

\paragraph{How much can locking save?}
Under iid calls, the certificate's trigger time obeys a law of large numbers whose only parameter is the margin of the modal answer.

\begin{theorem}[Lock-time law]
\label{thm:lock-time-law}
Let calls be iid with law $p$ on $\mathcal{U}_x\cup\{\bot\}$, $|\mathcal{U}_x|\geq2$, and let $\tau_B$ be the first certified prefix of \cref{eq:earliest-exact-stop} on the first $B$ calls. Suppose some answer $w^\star$ has $p_{w^\star}>p_a$ for every other answer $a$, and put $m=p_{w^\star}-\max_{a\neq w^\star}p_a$. (i) Almost surely as $B\to\infty$,
\begin{equation}
 \frac{\tau_B}{B}\longrightarrow\frac{1}{1+m},
 \qquad
 \frac{B-\tau_B}{B}\longrightarrow\frac{m}{1+m}\in\Bigl(0,\tfrac12\Bigr],
 \label{eq:lock-time-law}
\end{equation}
and the same limits hold for the open-universe rule. With a tied mode, $\tau_B/B\to1$. (ii) If the runner-up $a_2=\arg\max_{a\neq w^\star}p_a$ is unique and $\sigma^2=p_{w^\star}+p_{a_2}-m^2>0$, then $(\tau_B-B/(1+m))/\sqrt{B}\Rightarrow\mathcal{N}(0,\sigma^2/(1+m)^3)$, also for the open-universe rule. (iii) For every $B$ and every trajectory, $\tau_B\geq B/2$.
\end{theorem}

Asymptotically only the constraint against the runner-up binds: the certificate fires once $N+v_N(w^\star)-v_N(a_2)$, a random walk with drift $1+m$ and step variance $\sigma^2$, reaches $B$ up to a tie offset, which gives (i) and (ii) (proofs: \cref{app:lock-time-proof}). By (iii) and \cref{rem:earliest-exact-stopping}, no exact prefix-only rule saves more than half of the budget. By (i), locking nears that ceiling only when the mode is nearly certain. The law is gold-free and ties asymptotic locking to error competition: while gold remains the unique mode, merging wrong identities weakly lowers $m$ and so weakly raises the asymptotic lock-time fraction $1/(1+m)$. At finite budgets, predictions simulate the certificate under an independently estimated law and match the observed savings within 0.2 points on GSM-Symbolic ($B=16$ to 64) and 0.7 points on MMLU-Pro. On public GSM8K completions with 10,000 samples per problem from Llama-3-8B and 70B \citep{brown2024largelanguagemonkeys}, laws estimated from 2,000 held-out samples predict the savings within 0.1 points at every budget from 16 to 1,024, up to 39.6\% and 47.1\%. On problems with a confidently separated runner-up, the mean observed variance of $\tau_{1024}$ is 150 and 130 against 151 and 155 from (ii) (\cref{fig:lock-time-law} and \cref{app:lock-time-validation}). Public MATH has smaller estimated margins (means 0.24 and 0.45). Its savings level off at 15.8\% and 26.3\%, next to asymptotes of 15.75\% and 26.25\%, with predictions within 0.13 points at every budget (\cref{tab:lock-law-math}).

\begin{figure}[!htbp]
\centering
\includegraphics[width=\linewidth]{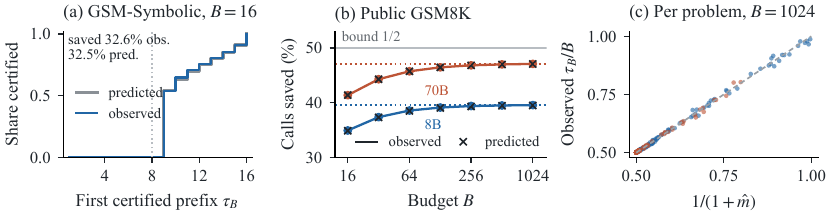}
\vspace{-0.3in}
\caption{Lock-time law. (a) GSM-Symbolic, $B=16$: first certified prefix predicted from an independently estimated law and observed (4,000 trajectories). (b) Public GSM8K, Llama-3-8B/70B-Instruct (127 problems): calls saved, observed (lines) and predicted (crosses), with the averages of $\hat m/(1+\hat m)$ (dotted). (c) Per problem at $B=1024$: observed $\tau_B/B$ against $1/(1+\hat m)$.}
\vspace{-0.1in}
\label{fig:lock-time-law}
\end{figure}

\paragraph{Physical execution.}
Charging every issued call, the certificate saves 26.2--36.6\% of sequential request time on 32 MMLU-Pro items per model and 9.6--21.6\% of batch-four request time on 48 fresh inputs per task, preserving every output.
Statistical stopping rules save more but change 1 to 22 of 1,200 native outputs, and fully parallel fixed-16 is faster but spends every call (\cref{app:physical-execution}). In a thinking-mode replay with near-unanimous answers, the certificate omits 41.9\% of generated tokens (\cref{app:thinking-tokens}).

\FloatBarrier

\section{Conclusion}

We studied fixed-budget LLM voting through the realized vote state and remaining budget. This view separates discovery, recoverability, endpoint selection, and decision locking. It yields sharp reachability thresholds and exact finite-horizon endpoint probabilities under a specified response law, and it exposes how recurring wrong-answer identities shape endpoint outcomes. Singleton reachability then provides a gold-free exact locking certificate that preserves the fixed-budget plurality output while saving inference calls.

\bibliography{references}
\bibliographystyle{iclr2027_conference}

\newpage
\appendix
\crefalias{section}{appendix}\crefalias{subsection}{subappendix}

\section{Related Work}
\label{sec:related}

\paragraph{Finite-budget sampling, error structure, and selection.}
Repeated sampling expands answer coverage \citep{brown2024largelanguagemonkeys,schaeffer2025monkeys}, and self-consistency aggregates repeated samples through the empirical answer mode \citep{wang2023selfconsistency,feng2026optimal,bay2026modalceiling}. Verifiers and candidate supply shape which generated answers can ultimately be selected \citep{stroebl2024flaws,dorner2026rocnreroll,hu2026oraclegap,ji2026candidate}, while selector-specific theory characterizes error decay under repeated inference \citep{chen2025provable,huang2026paradigms}. Permutation-based sampling perturbs the presentation of an input across calls \citep{tang2024permutation}, and the value of sampling diversity depends on the endpoint selector \citep{wang2026samplingdiversity}. Confidence-based routing can also fail under repeated sampling \citep{bahuguna2026backfires}. We make the realized vote state and remaining budget the primitive of a finite-horizon analysis, deriving a sharp candidate-wise reachability threshold, monotone contraction of the reachable-winner set into candidate-level conversion windows, exact endpoint probabilities under a specified response law, and a state-dependent pivotality characterization of how wrong-answer probability transfers affect endpoint success.

\paragraph{Exact decision locking and deterministic curtailment.}
Possible- and necessary-winner reasoning provides a framework for outcomes under incomplete votes \citep{conitzer2002elicitation,lang2020possible}. Classical ranking-and-selection work studies early termination of fixed-sample selection of the most probable multinomial cell \citep{bechhofer1959multinomial,bechhofer1995design}, with related optimality results for curtailed Bernoulli selection \citep{bechhofer1982closed,kulkarni1986optimal}. In sequential-analysis terminology, this is deterministic curtailment \citep{alling1963early}. We extend this structure to finite-budget LLM voting with budget-consuming abstentions, unseen answer identities, adaptive proposal policies, and weighted votes. Singleton reachability yields a gold-free certificate that preserves the full-budget plurality output for every admissible continuation, and for a known answer universe its first trigger is the earliest universally valid exact stopping time among prefix-only policies that must preserve the fixed-budget output. \Cref{thm:lock-time-law} further characterizes the resulting savings through a law of large numbers and fluctuation limit obtained from the first passage of a random walk against the runner-up \citep{gut2009stopped}.

\paragraph{Statistical stopping and inference-time efficiency.}
Existing stopping methods such as ASC, ESC, Bayesian stopping, EMS, and PETS reduce repeated sampling through consistency, mode-identification, agreement, or voting-specific stopping criteria \citep{aggarwal2023adaptive,li2024esc,huang2026bayesian,liu2026ems,liu2026pets}. Partial-trace methods use confidence filtering or probabilistic output-preservation guarantees \citep{fu2026deepconf,chen2026mars}. The endpoint probability $Q_r$ is a conditional-power quantity \citep{lan1988bvalue}, and thresholding an estimate of it connects directly to stochastic curtailment \citep{lan1982stochastically}. Because equal call counts can entail different prompt and model costs \citep{tran2026singleagent}, we evaluate stopping under sequential, microbatched, and fully parallel execution and report omitted calls, request time, and token savings separately. This places exact locking and statistical stopping on a common finite-horizon axis: singleton reachability provides an exact fixed-budget output-preservation boundary, while stochastic rules exchange output fidelity for additional savings.

\section{Finite-Budget Voting: Proofs and Probability Calculations}
\label{app:finite-budget}

This appendix develops the deterministic structure and probability calculations of \cref{sec:structure}. Remaining budget counts calls, and invalid responses append no vote. The locking experiments are reported in \cref{app:stopping-experiments}.

\subsection{Sharp Remaining-Call Threshold}
\label{app:sharp-threshold}

Use the count state and endpoint probability $Q_r(v;p)$ of \cref{sec:structure}. For each answer $y$, including answers not yet generated,
\[
 L_y(v)=\max\left\{0,1-v(y),
   \max_{a\ne y}\bigl(v(a)-v(y)+\delta_{a,y}\bigr)\right\},
 \qquad
 \delta_{a,y}=\mathbf{1}\{\kappa_x(a)>\kappa_x(y)\},
\]
so that $L(v)=L_{y^\star}(v)$ and $\delta_a=\delta_{a,y^\star}$. The presence term $1-v(y)$ prevents an ungenerated, zero-count answer from winning. Continuations are the admissible sequences in $(\mathcal{U}_x\cup\{\bot\})^m$ of \cref{sec:structure}, so an abstention consumes a call and adds no vote.

\begin{proof}[Proof of \cref{prop:reachability}(i)]
Let $X_a$ count the votes that a continuation of $m$ calls adds to answer $a$. If $y$ is the winner after these calls, then $v(y)+X_y\geq1$ and $X_y-X_a\geq v(a)-v(y)+\delta_{a,y}$ for every $a\neq y$. Since $X_a\geq0$, these inequalities give $X_y\geq L_y(v)$. Each call adds at most one vote, so $m\geq L_y(v)$. An intermediate prefix is itself a shorter continuation, so no continuation of fewer than $L_y(v)$ calls makes $y$ the winner at any point. Conversely, the length-$m$ continuation that votes only for $y$ leaves every rival count fixed and gives $X_y=m$, which meets all the inequalities whenever $m\geq L_y(v)$. Hence $\mathcal{R}_r(v)$ is exactly the set of answers that some continuation of the remaining $r$ calls makes the endpoint winner.
\end{proof}

For gold, let future calls be iid with law $p$, abstentions included. The number of future gold votes in $r$ calls is then $\mathrm{Bin}(r,p_{y^\star})$, and every gold win requires at least $L(v)$ of them, so
\begin{equation}
 Q_r(v;p)\leq\Pr\!\left[\mathrm{Bin}(r,p_{y^\star})\geq L(v)\right].
 \label{eq:finite-budget-tail}
\end{equation}
If $p_{y^\star}>0$ and $r\geq L(v)$, the all-gold continuation has probability $p_{y^\star}^{\,r}>0$ and selects gold. Hence $Q_r(v;p)>0$ exactly when $r\geq L(v)$. The binomial tail counts only the additional gold votes that every gold win requires. The recursion \cref{eq:abstention-recursion} also accounts for votes accumulated by competing answers.

At first discovery, $v_\star=1$, so $L(v)$ is the number of additional gold votes needed to catch the strongest tie-adjusted rival. The threshold separates impossible conversion from positive but potentially very small conversion probability.

\subsection{Reachable Winners and Exact Decision Locking}
\label{app:decision-locking}

Fix an answer universe $\mathcal{U}_x$, an injective tie priority $\kappa_x$, and a count vector $v$. Calls append at most one unit vote without changing earlier votes. Reachability ranges over all admissible continuations in $(\mathcal{U}_x\cup\{\bot\})^r$, whether or not a particular response process gives every sequence positive probability.

\begin{proof}[Proof of \cref{prop:reachability}(ii) and \cref{cor:decision-locking}]
For a nonempty current state, let $w$ be its winner. If $v(w)-v(a)\ge r+\delta_{a,w}$ for every $a\ne w$, each rival remains below $w$ under the count-and-priority order even if it receives all $r$ votes and $w$ receives none. Thus $w$ wins every continuation. If any inequality fails, the all-$a$ continuation puts that rival above $w$. The all-$w$ continuation still selects $w$, so at least two endpoint winners are reachable. This proves the equivalence with $\mathcal{R}_r(v)=\{w\}$.

Finally, any continuation of $v+e_a$ with $r-1$ calls is also a continuation of $v$ with $r$ calls after prepending $a$. Therefore
\[
 \mathcal{R}_{r-1}(v+e_a)\subseteq\mathcal{R}_r(v).
\]
An abstaining update leaves $v$ unchanged and consumes one call, and prepending that abstention gives $\mathcal{R}_{r-1}(v)\subseteq\mathcal{R}_r(v)$. This proves \cref{eq:shrinking-window} and \cref{prop:reachability}(ii) for every admissible update. Once the set is a singleton in a nonempty state, every subsequent prefix preserves its member as the winner through the fixed endpoint.
\end{proof}

\begin{proof}[Proof of the claim in \cref{rem:earliest-exact-stopping}]
Fix a realized prefix at time $N<\tau_B$. If $\mathcal{C}_N\neq\varnothing$, the locking condition has not yet occurred, so at least two distinct endpoint winners are reachable by admissible continuations from the same observed prefix. If $\mathcal{C}_N=\varnothing$, then $N<B$ and the answer universe is nonempty: an all-abstention continuation yields endpoint $\bot$, whereas assigning every remaining call to any fixed $y\in\mathcal{U}_x$ yields endpoint $y$. Thus in either case two admissible continuations from the same prefix have different budget-$B$ endpoints. A prefix-only policy that stops now must return one answer before observing which continuation occurs, so it cannot agree with both endpoints. Hence every policy required to preserve the full-budget endpoint for every admissible continuation must continue whenever $N<\tau_B$. At $N=\tau_B$, either the budget is exhausted or \cref{cor:decision-locking} guarantees that every admissible continuation has the same endpoint, so stopping is valid.
\end{proof}

Extra information restricting future responses can permit earlier stopping. For the open-answer replay, the conservative unseen-answer rule described below triggers at $\tau_B$ or one call later (\cref{lem:open-universe-gap}), so the GSM stopping times reported in \cref{app:stopping-replay} exceed the earliest lock by at most one call. The MMLU-Pro rule, which enumerates the full option universe, attains $\tau_B$ exactly.

\paragraph{When can a lock first occur?}
With at least two possible answer labels, the certificate cannot hold when $N<B-N$, because the leader has at most $N$ votes and some alternative could receive all $B-N$ remaining votes. For $B=16$, a lock at $N=8$ requires eight identical votes and a winner favored in ties against every other answer. With unknown unseen-answer priority, the conservative rule below first becomes eligible at $N=9$. These constraints explain the early zero rates in \cref{tab:decision-locking}.

\paragraph{Unseen answers, abstentions, and scope.}
For the open rational-answer task, only finitely many answers have appeared. We check \cref{eq:decision-locking} against every observed rival and additionally require $v(w)>r$. Any unseen answer has zero votes and can collect at most $r$, so this protects against every unseen competitor even if it would win a tie. For any finite universe that contains the observed answers, the rule is equivalent to the certificate for that universe together with $v(w)>r$. \Cref{lem:open-universe-gap} below shows that this requirement delays stopping by at most one call. Invalid completions consume a call and append no vote. They are retained in all denominators and remaining-call counts. The same inequalities remain sufficient when future calls may abstain, and all-valid continuations retain the necessity argument for a known answer universe. If no valid answer has appeared, the replay consumes the full budget and returns the baseline's abstention output.

\begin{lemma}[Cost of the unseen-answer check]
\label[lemma]{lem:open-universe-gap}
Let $|\mathcal{U}_x|\geq2$, and let $\tau_B^{\circ}$ be the first prefix at which the certificate holds with $v_N(w_N)>B-N$, or $B$ if there is none. On every trajectory, $\tau_B\leq\tau_B^{\circ}\leq\min(\tau_B+1,B)$. Moreover, $\tau_B^{\circ}=\tau_B+1$ exactly when $\tau_B<B$ and $v_{\tau_B}(w_{\tau_B})=B-\tau_B$.
\end{lemma}

\begin{proof}
The conservative rule adds a condition, so $\tau_B^{\circ}\geq\tau_B$, and both equal $B$ when $\tau_B=B$. Otherwise put $N=\tau_B$, $r=B-N\geq1$ and $w=w_N$. A rival $a\neq w$ exists because $|\mathcal{U}_x|\geq2$, and it gives $v_N(w)\geq v_N(a)+r+\delta_{a,w}\geq r$. If $v_N(w)>r$, both rules stop at $N$. If $v_N(w)=r$, the conservative rule fails at $N$. By \cref{prop:reachability}(ii), $\mathcal{R}_{r-1}(v_{N+1})\subseteq\{w\}$, and this set contains the winner of the nonempty state $v_{N+1}$. Hence prefix $N+1$ is certified with winner $w$. Since $v_{N+1}(w)\geq r>r-1$, the conservative rule stops at $N+1$. These three cases give both claims.
\end{proof}

The certificate fixes the answer interpretation, old votes, decision rule and total budget. A known response process with restricted support can have fewer reachable winners than the unrestricted set. No independence or response-law estimate is required for safety. The guarantee includes incorrect outputs and does not preclude gains from a different selector or a larger total budget. For example, with counts $(6,5,1)$, gold equal to the third answer and $r=4$, gold cannot recover, while the two leading wrong answers can still exchange the lead. Gold infeasibility alone does not certify decision locking.

\paragraph{Adaptive proposal choices.}
The same sufficient certificate applies when future calls choose their model, input representation or sampling law adaptively. The conditions are that every call appends at most one unit vote in the same canonical answer universe, earlier votes and tie priorities stay fixed, and the total call cap is unchanged. Conditional on any realized prefix, the adaptive policy's possible suffixes form a subset of the continuations already covered by the certificate. If that set of reachable winners is a singleton, stopping now preserves the capped adaptive policy's answer for every such suffix, including abstentions. This need not be the answer of a different policy whose earlier calls produced a different prefix. A selector that synthesizes answers, such as Borda fusion, or reweights earlier votes requires a separate argument. Thus the certificate can be combined with the proposal routers in \cref{app:future-protocols} without assuming that a cached base-policy suffix is the suffix those routers would actually generate.

\paragraph{Weighted votes.}
The same argument covers weighted plurality when each future call $t$ adds a weight $w_t\geq0$, fixed before the responses, to its answer and abstention adds nothing. Let $s$ be the accumulated scores, with only positive-score answers eligible, and $W_r=\sum_{t\leq r}w_t$ the remaining addable weight. Answer $y$ is reachable exactly when $s_y+W_r>0$, $s_y+W_r>s_j$ for every rival $j$ favored by the tie priority and $s_y+W_r\geq s_j$ for the others. A current winner $y$ is locked exactly when $s_y>s_j+W_r$ for every favored rival and $s_y\geq s_j+W_r$ for the others. Assigning all remaining weight to one answer attains each bound, and unit weights recover the unit-vote threshold. The lock preserves the weighted fixed-budget output, including errors, without independence assumptions. For an open universe, $s_y>W_r$ also covers unseen rivals. If weights may depend on the responses but lie in $[0,\bar w]$, the locking inequalities with $W_r=r\bar w$ remain sufficient. Retrospective rescoring and answer synthesis, such as Borda fusion, need separate arguments.

\subsection{Endpoint Laws, Abstentions, and First-Passage Events}
\label{app:endpoint-laws}

Let $\pi_x$ be a response law on $\mathcal{U}_x$ with no abstentions. For iid valid responses with law $\pi_x$, the exact endpoint probability is\nopagebreak
\begin{equation}
\begin{split}
 Q_r(v;\pi_x)
 =r!\sum_{z=L(v)}^r\frac{\pi_x(y^\star)^z}{z!}
 [u^{r-z}]
 \prod_{a\ne y^\star}
 \left(\sum_{j=0}^{\min\{r-z,\,v_\star+z-v(a)-\delta_a\}}
             \frac{\pi_x(a)^j u^j}{j!}\right),
\end{split}
 \label{eq:finite-budget-exact}
\end{equation}
where $[u^m]$ extracts the degree-$m$ coefficient and an empty outer sum is zero.

\begin{proof}[Endpoint formula and recursion]

For the exact expression, condition on $X_{y^\star}=z$. The remaining counts sum to $r-z$ and each wrong count is bounded above by $v_\star+z-v(a)-\delta_a$. Expanding the product in \cref{eq:finite-budget-exact} enumerates precisely these count vectors. Multiplication by $r!\pi_x(y^\star)^z/z!$ supplies their multinomial weights, without any independence assumption between pairwise wins.

For the recursion \cref{eq:abstention-recursion}, condition on the first future call. With probability $p_\bot$ it appends no vote and leaves $v$ unchanged, and with probability $p_y$ it moves the state to $v+e_y$. In both cases one call is spent and the remaining calls keep their iid law.

\end{proof}

The exact calculation accounts for rival votes during the remaining calls. Truncated polynomial convolution evaluates it in $O(|\mathcal{U}_x|r^3)$ arithmetic operations.

\paragraph{Endpoint and first-passage events.}
Write $A(v)$ for the current gold-selection indicator and $H_r(v;\pi_x)$ for selection at least once from the current state through $r$ future calls. Their recursions differ, and the first specializes \cref{eq:abstention-recursion} to a law without abstentions:
\begin{align*}
 Q_0(v;\pi_x)&=A(v), &
 Q_r(v;\pi_x)&=\sum_y\pi_x(y)Q_{r-1}(v+e_y;\pi_x),\\
 H_0(v;\pi_x)&=A(v), &
 H_r(v;\pi_x)&=A(v)+(1-A(v))\sum_y\pi_x(y)H_{r-1}(v+e_y;\pi_x).
\end{align*}
Here $e_y$ adds one vote for $y$. The second recursion is absorbing, while the first permits later reversals. For two answers with counts $(1,2)$, gold favored in ties, equal future response probabilities, and $r=2$, $H_2=1/2$ while $Q_2=1/4$. Thus a discovery-to-first-selection delay distribution cannot be used as an endpoint accuracy curve.

\paragraph{Abstention mixture.}
Suppose each further call abstains with probability $p_\bot$ and otherwise returns $y$ with probability $\pi_x(y)$, independently of the past, so that $p_y=(1-p_\bot)\pi_x(y)$ is the unconditional probability that one call adds a vote for $y$. The endpoint depends on the calls only through their valid answers. The number $K$ of valid calls is $\mathrm{Bin}(r,1-p_\bot)$, and given $K=k$ the valid answers are iid with law $\pi_x$. Conditioning on $K$ gives
\[
 Q_r(v;p)=\sum_{k=0}^{r}\binom{r}{k}(1-p_\bot)^k p_\bot^{\,r-k}\,Q_k(v;\pi_x),
\]
where $Q_k(v;\pi_x)$ follows the valid-response recursion above or \cref{eq:finite-budget-exact}.

\subsection{The Lock-Time Law: Fluctuations and Proofs}
\label{app:lock-time-proof}

Throughout, $Z_1,Z_2,\ldots$ are iid with unconditional law $p$ on $\mathcal{U}_x\cup\{\bot\}$, $|\mathcal{U}_x|\geq2$, and an abstention consumes a call and adds no vote. Let $V_N(y)$ be the count of $y$ after $N$ calls and let $r=B-N$. For each budget $B$, $\tau_B$ is the first certified prefix of \cref{eq:earliest-exact-stop} on $Z_1,\ldots,Z_B$ (with $\tau_B=B$ when no prefix is certified), and $\tau_B^{\circ}$ is the open-universe variant that additionally requires $V_N(w_N)>B-N$. By \cref{cor:decision-locking}, prefix $N$ is certified with winner $w$ if and only if
\begin{equation}
 V_N(w)-V_N(a)\;\geq\;B-N+\delta_{a,w}\qquad\text{for all }a\in\mathcal{U}_x\setminus\{w\},
 \label{eq:lock-inequalities}
\end{equation}
where $\delta_{a,w}=\mathbf{1}\{\kappa_x(a)>\kappa_x(w)\}\in\{0,1\}$. The open-universe variant adds $V_N(w)\geq B-N+1$. On a trajectory whose first $B$ calls all abstain the pool stays empty, no prefix is certified, the budget is exhausted at the terminal prefix, and $\tau_B=B$ by convention.

In \cref{thm:lock-time-law}(ii), a unique runner-up means $p_{a_2}=\max_{a\neq w^\star}p_a>p_a$ for every $a\in\mathcal{U}_x\setminus\{w^\star,a_2\}$, and the limit reads
\begin{equation}
 \frac{\tau_B-B/(1+m)}{\sqrt{B}}\;\Longrightarrow\;\mathcal{N}\!\Bigl(0,\;\frac{\sigma^2}{(1+m)^3}\Bigr),
 \label{eq:lock-time-clt}
\end{equation}
also for $\tau_B^{\circ}$. In the excluded case $\sigma^2=0$, $p_{w^\star}=1$ and $\tau_B=\lceil(B+\delta)/2\rceil$ exactly, where $\delta=\mathbf{1}\{\text{some }a\neq w^\star\text{ has }\kappa_x(a)>\kappa_x(w^\star)\}$.

\begin{proof}[Proof of \cref{thm:lock-time-law}(i) and (iii)]
All statements concern one infinite sequence $Z_1,Z_2,\ldots$, and $\tau_B$ is computed on its first $B$ terms.

\emph{Deterministic bound.} Pick $a\neq w$. Since $V_N(w)\leq N$ and $V_N(a)\geq0$, \cref{eq:lock-inequalities} forces $N\geq B-N$, so no prefix $N<B/2$ is certified and $\tau_B\geq B/2$ on every trajectory.

\emph{Uniform law of large numbers.} By the strong law, $V_N(y)/N\to p_y$ for every $y\in\mathcal{U}_x\cup\{\bot\}$ almost surely. Since the alphabet is finite, for every $\varepsilon>0$ there is almost surely a finite $N_0$ with $|V_N(y)-Np_y|\leq\varepsilon N/4$ for all $N\geq N_0$ and all $y$. Since every certified prefix satisfies $N\geq B/2$, for $B\geq2N_0$ these bounds hold at every prefix that matters below.

\emph{Unique leadership.} Let $m>0$ and $0<\varepsilon<m$. For $N\geq\max(N_0,B/2)$ and every $a\neq w^\star$, $V_N(w^\star)-V_N(a)\geq N(p_{w^\star}-p_a)-\varepsilon N/2\geq N(m-\varepsilon/2)>1$ once $N>2/(m-\varepsilon)$, so $w^\star$ is the plurality winner at every such prefix and a certificate, if any, has winner $w^\star$.

\emph{Upper bound.} Put $N_\varepsilon=\lceil B/(1+m-\varepsilon)\rceil\in[B/2,B]$ for $\varepsilon<m$. At $N=N_\varepsilon\geq N_0$ and for every $a\neq w^\star$,
\[
 V_N(w^\star)-V_N(a)-\delta_{a,w^\star}\;\geq\;N(m-\varepsilon/2)-1\;\geq\;B-N,
\]
because $N(1+m-\varepsilon)\geq B$ leaves slack $N\varepsilon/2-1>0$ for large $B$. Hence prefix $N_\varepsilon$ is certified and $\limsup_B\tau_B/B\leq1/(1+m-\varepsilon)$.

\emph{Lower bound.} Let $a_2$ attain $\max_{a\neq w^\star}p_a$ and put $N'_\varepsilon=\lfloor B/(1+m+\varepsilon)\rfloor$. For every $N\in[B/2,N'_\varepsilon]$ with $N\geq N_0$, the winner is $w^\star$ and $V_N(w^\star)-V_N(a_2)\leq N(m+\varepsilon/2)<B-N$ because $N(1+m+\varepsilon)\leq B$. Thus \cref{eq:lock-inequalities} fails for $a=a_2$ and no such prefix is certified. Prefixes below $B/2$ are never certified. Hence $\liminf_B\tau_B/B\geq1/(1+m+\varepsilon)$. Letting $\varepsilon\downarrow0$ proves \cref{eq:lock-time-law}. The second limit is $1-1/(1+m)$ and $m\leq1$ gives the range $(0,1/2]$.

\emph{Tied modes.} Let $M=\{y\in\mathcal{U}_x:p_y=\max_{a}p_a\}$ with $|M|\geq2$. For $w\notin M$ and any $w_1\in M$, $V_N(w)-V_N(w_1)\leq-N(p_{w_1}-p_w)/2<0$ for all $N\geq\max(N_0,B/2)$ with $\varepsilon$ small, so $w$ is never the winner at a prefix that could be certified. For $w\in M$ choose $w'\in M\setminus\{w\}$. Since $p_w=p_{w'}$, $V_N(w)-V_N(w')\leq\varepsilon N/2$ for $N\geq N_0$. A certificate with winner $w$ at prefix $N\geq\max(N_0,B/2)$ requires $V_N(w)-V_N(w')\geq B-N$, hence $N\geq B/(1+\varepsilon/2)$. Thus $\liminf_B\tau_B/B\geq1/(1+\varepsilon/2)$ for every small $\varepsilon>0$, and $\tau_B/B\to1$.

\emph{Open universe.} \Cref{lem:open-universe-gap} gives $\tau_B\leq\tau_B^{\circ}\leq\tau_B+1$, so $\tau_B^{\circ}/B$ has the same limits as $\tau_B/B$, including the tied-mode limit.
\end{proof}

\begin{proof}[Proof of \cref{thm:lock-time-law}(ii)]
If $\sigma^2=0$ the increments $\mathbf{1}\{Z_i=w^\star\}-\mathbf{1}\{Z_i=a_2\}$ are constant, which forces $p_{w^\star}=1$. Then $V_N(w^\star)=N$, every rival has count zero, and \cref{eq:lock-inequalities} reads $N\geq B-N+\delta$, giving $\tau_B=\lceil(B+\delta)/2\rceil$.

Let $\sigma^2>0$ and $S_N=N+V_N(w^\star)-V_N(a_2)$, a random walk with iid increments $X_i=1+\mathbf{1}\{Z_i=w^\star\}-\mathbf{1}\{Z_i=a_2\}\in\{0,1,2\}$, mean $\mu=1+m$ and variance $\sigma^2$. The inequality \cref{eq:lock-inequalities} for $a=a_2$ and winner $w^\star$ reads $S_N\geq B+\delta_{a_2,w^\star}$. Let $T_B=\min\{N:S_N\geq B+\delta_{a_2,w^\star}\}$.

We first show that $\tau_B=T_B$ for all large $B$ almost surely. Use the uniform law of large numbers from the proof of (i) with $\varepsilon$ smaller than $m$ and than $p_{a_2}-p_a$ for every $a\in\mathcal{U}_x\setminus\{w^\star,a_2\}$. For all large $N$, $w^\star$ is the plurality winner, and every third rival $a$ satisfies
\[
 V_N(w^\star)-V_N(a)-\delta_{a,w^\star}\;\geq\;V_N(w^\star)-V_N(a_2)-\delta_{a_2,w^\star}+N(p_{a_2}-p_a-\varepsilon/2)-1,
\]
where the last two terms sum to a positive number. Hence the $a_2$-inequality implies every other inequality of \cref{eq:lock-inequalities}. Prefixes below $B/2$ are never certified and have $S_N\leq2N<B$, so for all large $B$ a prefix is certified exactly when $S_N\geq B+\delta_{a_2,w^\star}$. Moreover $S_B/B\to1+m>1$, so $T_B\leq B$ for all large $B$, and then $\tau_B=T_B$.

The increments are nonnegative, so $S_N$ is nondecreasing and $\Pr(T_B\leq n)=\Pr(S_n\geq B+\delta_{a_2,w^\star})$ for every $n$. Fix $x\in\mathbb{R}$ and put $n_B=\lfloor B/\mu+x\sqrt{B}\rfloor$. Then $\mu n_B=B+\mu x\sqrt{B}+O(1)$ and $n_B/B\to1/\mu$, so the central limit theorem for $S_{n_B}$ gives
\[
 \Pr\!\left(\frac{T_B-B/\mu}{\sqrt{B}}\leq x\right)=\Pr\!\left(S_{n_B}\geq B+\delta_{a_2,w^\star}\right)\;\longrightarrow\;\Phi\!\left(\frac{x\mu^{3/2}}{\sigma}\right),
\]
where $\Phi$ is the standard normal distribution function. The limit is the $\mathcal{N}(0,\sigma^2/\mu^3)$ distribution function at $x$. Because $\tau_B=T_B$ eventually, the same limit holds for $\tau_B$, which proves \cref{eq:lock-time-clt}. By \cref{lem:open-universe-gap}, $(\tau_B^{\circ}-\tau_B)/\sqrt{B}\to0$, so the limit also holds for $\tau_B^{\circ}$.
\end{proof}

\paragraph{Scope.}
The law concerns a fixed item under iid calls from a finite answer universe, uses no gold label, and holds for a wrong mode as for a correct one. For an open numeric task the realized alphabet is finite but not fixed in advance. The almost-sure limits extend to a countable universe because the empirical frequencies then converge uniformly (the total variation between the empirical and the true law tends to zero almost surely). The fluctuation statement is proved here for finite universes only. Dependent or adaptive responses fall outside the law, although the certificate itself (\cref{cor:decision-locking}) still covers them. Finite-budget predictions in \cref{app:lock-time-validation} are Monte Carlo approximations of the finite-budget plug-in distribution, with fixed seeds and reported simulation error.

\section{Error Competition: Proofs and Counterexamples}

\label{app:concentration-design}

\subsection{Controlled Merging and Splitting}

\begin{proof}[Proof of \cref{prop:wrong-merging} and of $L(v^{\mathrm{merge}})\geq L(v)$]
For every wrong block $S$, its merged count $\sum_{a\in S}v(a)$ is at least each member's count, and its tie advantage over gold is at least each member's. Gold and abstention positions are unchanged. Hence a gold winner against every merged block beats every original member. Likewise, each original tie-adjusted deficit is bounded by the merged block's deficit. The presence term is unchanged, so taking their maxima gives $L(v^{\mathrm{merge}})\ge L(v)$. This holds prefix by prefix on each coupled sequence and therefore also after averaging under any sequence law.
\end{proof}

\Cref{prop:wrong-merging} compares response sequences through an explicit coupling. For a finite iid law with gold mass $p_\star=0.3$, abstention mass $p_\bot=0.1$, and the remaining $0.6$ divided equally among $K$ wrong identities, successive doubling of $K$ splits wrong identities without changing the gold or abstention events. Give every wrong identity priority over gold in a count tie. Exact coefficient extraction from \cref{app:endpoint-laws}, including abstentions, yields 16-call accuracies $7.43\%,27.65\%,53.02\%,71.85\%,83.04\%,88.66\%$ for $K=1,2,4,8,16,32$ (\cref{fig:concentration-math}a). Coverage remains $1-0.7^{16}=99.67\%$. This is a response-law calculation, with no claim that an input permutation implements the coupling.

\begin{figure}[htbp]
\centering
\includegraphics[width=\linewidth]{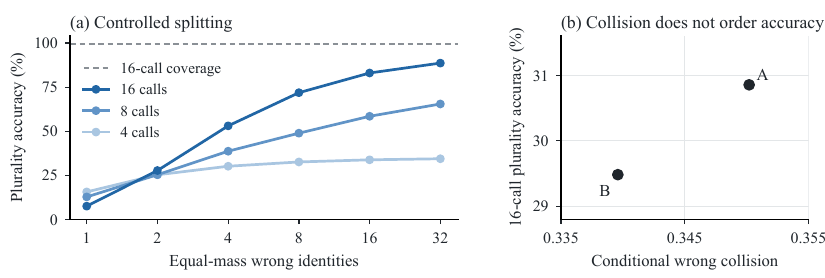}
\caption{Exact response-law diagnostics. (a) Splitting equal-mass wrong identities preserves single-call correctness and coverage at each budget and improves plurality. Here $p_\star=0.3$ and $p_\bot=0.1$. The dashed line shows coverage at 16 calls. (b) Arbitrary laws need not obey a scalar collision order. Laws A and B below both have $p_\star=0.3$ and no abstention, yet B has lower collision and lower accuracy. Gold loses every count tie in both panels.}
\label{fig:concentration-math}
\end{figure}

For the counterexample, use $p_\star=0.3$, no abstention, and unconditional wrong-mass vectors
\[
 \text{A}: (0.29,0.27,0.12,0.01,0.01),\qquad
 \text{B}: (0.37,0.11,0.10,0.07,0.05).
\]
Conditional collision is $\sum_j(p_j/0.7)^2$: $0.350204$ for A and $0.339592$ for B. With 16 iid calls and gold losing ties, their exact plurality accuracies are $0.308548$ and $0.294812$. B has lower collision but a larger most-probable rival (\cref{fig:concentration-math}b). The vectors are not linked by the splitting coupling. These rational-mass examples were computed by exact arithmetic.

\subsection{Necessary-Condition Filtering Preserves Correct Wins}
\label{app:necessary-filter}

Let $A_x\subseteq\mathcal{U}_x$ be a fixed set specified from the input, with every correct answer in $A_x$. Membership is a necessary condition for correctness. An accepted candidate may still be wrong. Replace each response outside $A_x$ by an abstention, retain the canonical identities and fixed tie order, and charge every response against the original call cap. Exact multiset membership for sorting and legal-path membership for a shortest-path problem are examples. Neither predicate checks the ordering or optimality of an accepted answer.

\begin{proposition}[Filtering necessary conditions]
\label[proposition]{prop:necessary-filter}
Fix a realized response sequence and a correct answer $y\in A_x$. At every prefix, filtering preserves the presence and count of $y$, cannot destroy a plurality win by $y$, and satisfies
\[
 L_y(v^{A_x})\leq L_y(v).
\]
These conclusions require no independence assumption. At the fixed cap, falling back to raw plurality when the filtered pool is empty also preserves every raw correct win.
\end{proposition}
\begin{proof}
Gold's count stays fixed. Every retained rival keeps its count and tie priority, and removed rivals have zero count. Hence no adjusted rival deficit in the recovery-threshold formula increases, and the presence term is unchanged. If $y$ won before filtering it is present, so the filtered pool is nonempty and its winner remains $y$. If the filtered pool is empty, the raw winner could not have been correct, because every correct answer would have been retained.
\end{proof}

The original decision-locking certificate applies to a nonempty filtered state, with $r=B-N$ determined by all issued calls, including rejected ones. It preserves the full-cap \emph{filtered} endpoint, including its errors. The empty-pool fallback is invoked only at cap. Any learned filter that can reject a correct answer falls outside this proposition's assumption and requires an additional error analysis.

Filtering need not make exact locking earlier. With cap 16 and the conservative unseen-rival check, consider nine copies of one inadmissible wrong answer followed by seven copies of gold. Raw plurality locks the wrong answer at call 9. Filtering yields nine abstentions, then locks gold at call 13. The selector improves the endpoint while its first safe stopping time becomes later. Cost comparisons must therefore specify which full-cap selector each certificate preserves.

Conditional wrong collision can increase during this beneficial intervention. Consider iid probabilities $(p_y,p_a,p_b,p_c)=(0.3,0.2,0.4,0.1)$ and retain only $y,a$. For sorting the input \texttt{ant bee}, these candidates can be \texttt{ant bee}, \texttt{bee ant}, \texttt{ant}, and \texttt{ant cat}, so the filter is input-defined and never inspects gold. Gold loses every tie. Conditional wrong collision rises from $3/7$ to $1$, while unconditional wrong collision falls from $0.21$ to $0.04$. At budget 16, exact plurality accuracy increases from $23.1019\%$ to $65.4329\%$, while coverage remains $1-0.7^{16}=99.6677\%$. This example explains why error mass, identity counts and the conditioning denominator must accompany a concentration statistic.

\subsection{State-Dependent Probability Transfers}
\label{app:state-competition}

Fix the alphabet $\mathcal{U}_x\cup\{\bot\}$, put $e_\bot=0$, and use unconditional call probabilities. For any candidate $y$, let $A_y(v)$ indicate its selection and $Q_{r,y}(v;p)=\mathbb{E}_p[A_y(v+\sum_{t=1}^r e_{Z_t})]$ for iid calls. Gold is an offline label in wrong-identity comparisons.

\begin{proof}[Proof of \cref{prop:pivotal-transfer}]
Put $p_\epsilon=p+\epsilon(e_b-e_a)$ on its legal probability interval, $W=v+\sum_{t=1}^{r-1}e_{Z_t}$, $\delta_j=\mathbf{1}\{\kappa_x(j)>\kappa_x(y^\star)\}$ and $E_j=\{A_{y^\star}(W)=1,\ W_{y^\star}-W_j-\delta_j=0\}$, the event that rival $j$ is pivotal. For $r\geq1$ we show
\begin{align}
 \frac{d}{d\epsilon}Q_{r,y^\star}(v;p_\epsilon)
 &=r\{Q_{r-1,y^\star}(v+e_b;p_\epsilon)-Q_{r-1,y^\star}(v+e_a;p_\epsilon)\}\notag\\
 &=r\{\Pr_{p_\epsilon}(E_a)-\Pr_{p_\epsilon}(E_b)\}.
 \label{eq:pivotal-transfer}
\end{align}
Differentiate the ordered-sequence probability polynomial. Each of its $r$ positions gives the same expectation after replacing that position by $b$ or $a$. Adding a wrong vote cannot create a gold win. It removes an existing win exactly when that rival has zero tie-adjusted slack. Thus $A_{y^\star}(W+e_b)-A_{y^\star}(W+e_a) =\mathbf{1}_{E_a}-\mathbf{1}_{E_b}$. Polynomial derivatives also give the appropriate one-sided derivatives at simplex boundaries.
\end{proof}

\begin{lemma}[Conditional binomial sign threshold]
\label[lemma]{lem:transfer-threshold}
Keep $s=p_a+p_b>0$ and other probabilities fixed, and put $t=p_a/s$. Condition on all future counts outside $\{a,b\}$, leaving $m$ calls of these two types, and let $w$ include the conditioned votes. Discard a stratum if gold is absent or loses to an outside rival. Otherwise set
\[
 C_a=w_{y^\star}-w_a-\delta_a,\quad C_b=w_{y^\star}-w_b-\delta_b,
 \qquad \ell=\max(0,m-C_b),\quad u=\min(m,C_a).
\]
For $\ell\leq u$, conditional success is $F(t)=\Pr\{\ell\leq\mathrm{Bin}(m,t)\leq u\}$. For $m\geq1$,
\begin{equation}
 \frac{dF}{d\epsilon}=\frac{m}{s}
 \{f_{m-1,t}(u)-f_{m-1,t}(\ell-1)\},
 \label{eq:conditional-transfer}
\end{equation}
where $f$ is a binomial point probability, zero outside its support. For $1\leq\ell\leq u\leq m-1$ and $0<t<1$, the sign is nonnegative exactly when
\begin{equation}
 t\geq t_{m,\ell,u},\qquad
 \frac{t_{m,\ell,u}}{1-t_{m,\ell,u}}
 =\left[\frac{\binom{m-1}{\ell-1}}{\binom{m-1}{u}}\right]^{1/(u-\ell+1)}.
 \label{eq:transfer-odds}
\end{equation}
The edge cases $\ell=0,u<m$ and $\ell>0,u=m$ have nonnegative and nonpositive derivatives, respectively, throughout $0\leq t\leq1$. At $t=0,1$, use \cref{eq:conditional-transfer} directly, as both point probabilities can vanish. Empty and full intervals are constant. Outer stratum weights are unchanged by the transfer. Their weighted derivative sum gives the exact unconditional sign. For a forward transfer $0\leq\epsilon\leq p_a$, nonnegative derivatives in every positive-weight stratum throughout $[0,\epsilon]$ suffice for a weak improvement, and nonpositive derivatives suffice for a weak decrease. Check interior strata by their thresholds and edge strata by \cref{eq:conditional-transfer}.
\end{lemma}

\begin{proof}
Conditionally, $X_a\sim\mathrm{Bin}(m,t)$ and $X_b=m-X_a$. Gold wins exactly when $\ell\leq X_a\leq u$. Differentiate the difference of the two binomial cumulative probabilities, using $dt/d\epsilon=-1/s$. For interior indices the ratio of the two point probabilities is $\binom{m-1}{u}\binom{m-1}{\ell-1}^{-1}(t/(1-t))^{u-\ell+1}$, which is increasing and equals one at the stated threshold. Edge cases follow from the zero point probabilities. Collapsing $a,b$ into a single symbol shows that outside weights depend only on $s$ and the other probabilities. Sum the derivatives and integrate.
\end{proof}

This conditions on extra wrong identities as well, so it applies beyond three answers. Opposing signs across strata require their weighted sum, while the simultaneous inequalities are sufficient. With $v=(2,0,1)$, $r=1$ and gold losing ties, changing $(0.3,0.6,0.1)$ to $(0.3,0.5,0.2)$ changes accuracy from $9/10$ to $4/5$. At $v=(2,1,0)$ the same transfer changes accuracy from $2/5$ to $1/2$.

\section{Experimental Setup}
\label{app:experimental-overview}
\label{app:study-overview}

The empirical studies are organized around three questions: whether finite-budget recovery dynamics appear across response protocols, whether the realized vote state predicts the fixed-budget endpoint, and whether changing error competition or stopping behavior changes the final decision. Table~\ref{tab:study-map} summarizes the datasets and the role each cohort plays. Calls are counted against the budget even when a response is invalid and contributes no vote. Unless stated otherwise, uncertainty intervals are pointwise 95\% bootstrap intervals with the item, or the GSM template family, as the resampling unit, so repeated calls and model-specific records from the same unit remain grouped.

\begin{table}[!htbp]
\centering\small\setlength{\tabcolsep}{4pt}
\caption{Empirical studies and their role in the paper. The cohorts are analyzed separately.}
\label{tab:study-map}
\begin{tabularx}{\linewidth}{@{}>{\raggedright\arraybackslash}p{.18\linewidth}>{\raggedright\arraybackslash}p{.32\linewidth}>{\raggedright\arraybackslash}X>{\raggedright\arraybackslash}X@{}}
\toprule
Study & Purpose & Data & Scale\\
\midrule
Forced-choice MMLU-Pro & Illustrate recovery trajectories and conversion-window closure & 200 items, five models, two sequences & 2,000 trajectories\\
GSM-Symbolic & Prospective endpoint forecasting from a realized call-8 state & 250 items in 100 template families, five models & 5,000 forecast trajectories\\
Native generation & Test recovery dynamics under direct answer generation & 1,200 MMLU-Pro and 200 Word16 items, Qwen3.5-2B/4B & 2,800 trajectories\\
Diversification & Test representation diversity and selector interaction & 256 held-out Word16 and 256 held-out MMLU-Pro items per task & 16 calls per policy\\
Physical stopping & Measure request-time consequences of exact locking & 32 MMLU-Pro items per checkpoint & batch sizes 1 and 4\\
Judge selector & Select after plurality's window has closed & Call-12 states of 200 MMLU-Pro items, five pools & 2,000 states, 268 vote-infeasible\\
Effect forecasts & Predict an intervention's effect from the realized state & 128 MMLU-Pro inputs, Qwen3.5-4B, Base and Permuted & 256 states\\
Fresh native cohort & Replicate filtering and selector interaction, and test physical stopping & 384 test and 48 online inputs per task: Word16, MMLU-Pro, shortest paths & 16 calls per arm and 960 requests\\
Public GSM8K & Lock-time law, gap and forecasts at large budgets & 127 problems, 10,000 samples each, Llama-3-8B/70B-Instruct & budgets 16 to 1,024\\
Public MATH & Gap, forecasts and lock-time law on a harder open-answer task & 128 problems, 10,000 samples each, Llama-3-8B/70B-Instruct & budgets 16 to 1,024\\
Forced-choice protocol, 27B & Failure composition and locking at a larger scale & 200 MMLU-Pro items, Qwen3.5-27B, two sequences & 400 trajectories\\
Temperature and permutation & Test permutation at two temperatures on new inputs & 768 MMLU-Pro test inputs, Qwen3.5-4B & 16 calls per cell\\
Thinking-mode tokens & Token savings of exact locking with reasoning models & 100 MATH-500 problems, Qwen3-14B/8B in thinking mode & two 32-call sequences per problem\\
\bottomrule
\end{tabularx}
\end{table}

\paragraph{Protocols.}
Forced-choice MMLU-Pro turns each call into one vote through rationale generation, 64-word processing and forced-choice scoring. The 16-call budget counts generations, each of which requires an additional scoring forward pass, and an alternative rationale-processing rule is examined as a sensitivity (\cref{app:historical-protocol}). In this cohort, Ministral's chat markers are encoded as ordinary text. A check with corrected encoding is inconclusive, which precludes capability rankings involving Ministral. GSM-Symbolic uses native chat encoding and a fixed rational-answer parser with template family as the resampling unit. Of its 6,250 trajectories, 5,000 enter forecast evaluation, and the remaining 1,250 estimate response laws. Native generation parses a generated final answer, with no rescoring or repair, under native chat templates, nonthinking generation and a 4,096-token cap, for 44,800 calls. In every protocol, invalid responses consume calls without adding votes. Public MATH samples are parsed with the dataset grader's own final-answer extraction and normalization, followed by a deterministic identity rule (\cref{app:scope-scale}). \Cref{tab:settings} lists the generation settings of every protocol.

\begin{table}[!htbp]
\centering\footnotesize\setlength{\tabcolsep}{3pt}
\caption{Generation settings by protocol. In-house generation uses BF16 weights with Hugging Face Transformers, except the thinking-mode study, which uses vLLM 0.13.0. Every invalid call consumes budget and adds no vote.}
\label{tab:settings}
\begin{tabularx}{\linewidth}{@{}l*{5}{>{\raggedright\arraybackslash}X}@{}}
\toprule
 & Forced-choice MMLU-Pro$^{a}$ & GSM-Symbolic & Native$^{b}$ & Thinking mode & Public GSM8K, MATH\\
\midrule
Models & Qwen3.5-0.8B/2B/4B/27B, Gemma-4-E2B, Ministral-3-3B & the five 0.8B--4B checkpoints & Qwen3.5-2B/4B & Qwen3-14B/8B & Llama-3-8B/70B-Instruct\\
Prompting & chat, nonthinking, $\leq$40-word rationale & native chat, nonthinking & native chat, nonthinking & native chat, thinking & few-shot completion (5 GSM8K, 4 MATH)\\
Temperature & 0.7 & 0.7 & 0.7 (High temp.\ 1.0) & 0.6 & 0.6\\
top-$p$ / top-$k$ & 0.9 / model default & 0.9 / 20 & 0.9 / 20 & 0.95 / 20 & none / none\\
Repetition penalty & 1.05 & 1.05 & 1.05 & none & none\\
Max new tokens & 96 & 1{,}024 & 4{,}096$^{c}$ & 16{,}384 & 512\\
Vote & argmax forced-choice letter & last boxed exact rational & letter, word sequence or path & final boxed answer & dataset extractor\\
Budgets $B$ & 16 & 16--64 & 16 & 16--32 & 16--1{,}024\\
\bottomrule
\end{tabularx}
\par\smallskip\raggedright\footnotesize $^{a}$Also the 27B replay and the judge study's pools. Judge rationales use the same settings. $^{b}$Native MMLU-Pro and Word16, the diversification study, the effect forecasts, the fresh three-task cohort, the fresh MMLU-Pro temperature study and the physical stopping runs. $^{c}$The fresh MMLU-Pro temperature study reaches 4,096 tokens by staged continuation.
\end{table}

\paragraph{Common state and failure definitions.} Feasibility is evaluated immediately after the first correct vote is added. With budget 16 and first gold discovery at call $t$, covered-but-unselected endpoints are partitioned into three classes: I (infeasible at discovery, $L(v_t)>16-t$), F (feasible at discovery but never selected at any prefix) and L (selected at least once and then lost at the endpoint). Gold labels are used only for offline evaluation of these state categories. Routing and stopping rules do not use them.

\paragraph{Prospective and retrospective analyses.} Forecasts, selectors and primary contrasts were fixed before the outcomes were generated or opened in seven studies. These are the GSM forecasts, the effect forecasts, the Word16 diversification study, the fresh native cohort, the judge study, the fresh MMLU-Pro temperature study and the thinking-mode study. The forced-choice recovery analyses and the public-data, 27B and longer-law replays are retrospective. The second GSM wave was added after an inconclusive first-wave calibration, so pooled waves are reported descriptively. The lock-time predictions were stored before the replay analyses, although the $B=16$ lock times of our own cohorts were already known. Checkpoint revisions, prompts, parsers and software settings are listed in \cref{app:checkpoints}.

\FloatBarrier

\section{Recovery Dynamics Across Protocols}
\label{app:historical-study}

\subsection{Forced-Choice MMLU-Pro}
\label{app:historical-protocol}
\label{app:confirmatory-protocol}
\label{sec:confirmatory-design}
\label{sec:contemporary-validation}
\label{app:ext-contemporary}

The forced-choice cohort contains 200 MMLU-Pro items \citep{wang2024mmlupro}, five checkpoints, and two independently seeded 16-call sequences per model-item pair. Each call generates a concise rationale that is converted to a forced-choice vote using the fixed 64-word scoring rule. The analyses of this cohort are retrospective, so it illustrates the state dynamics developed in \cref{sec:structure}, and the native cohort below provides a separate protocol check.

\paragraph{Conversion-window dynamics.} Among the 1,384 trajectories in which gold appears, 74 are already infeasible at discovery, 433 lose reachability later but before the final call, 46 remain reachable through call 15 and fail at the endpoint, and 831 select gold at call 16. \Cref{fig:gap}b shows the same phenomenon call by call for the 553 covered failures: recoverability can disappear after discovery, and a temporarily selected answer can subsequently lose. Across all 2,000 trajectories, coverage rises from 35.95\% after one call to 69.20\% after sixteen, while current selection rises only from 35.95\% to 41.55\% (\cref{fig:historical-evolution}).

\begin{table}[!htbp]
\centering\small
\caption{Timing of the gold candidate's conversion window in forced-choice MMLU-Pro. The denominator is the 1,384 trajectories in which gold is discovered.}
\label{tab:historical-window-audit}
\begin{tabular}{@{}lrr@{}}
\toprule
Conversion-window outcome & Count & Share of discovered (\%)\\
\midrule
Infeasible at discovery & 74 & 5.3\\
Closes later, before final call & 433 & 31.3\\
Reachable through call 15, endpoint failure & 46 & 3.3\\
Open through endpoint, gold selected & 831 & 60.0\\
\bottomrule
\end{tabular}
\end{table}

\begin{figure}[!htbp]
\centering
\includegraphics[width=0.6\linewidth]{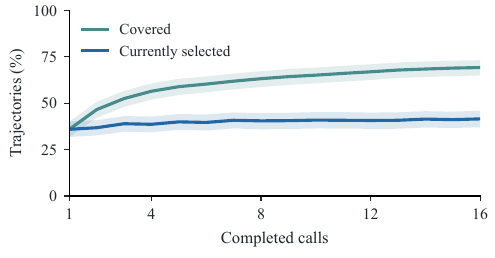}
\caption{Forced-choice MMLU-Pro coverage and current selection across all 2,000 trajectories, with pointwise 95\% item-bootstrap intervals. The per-call states of the 553 covered failures appear in \cref{fig:gap}b.}
\label{fig:historical-evolution}
\end{figure}

\paragraph{Retrospective forecast check.} Using the opposite 16-call sequence to estimate the response law, the exact finite-horizon calculation predicts an endpoint gold-selection rate of 0.1327 among the 454 call-8 states that are covered but unselected, against an observed rate of 0.1300. Its Brier score is 0.02959, compared with 0.04665 for a law-only predictor and 0.04850 for persistence, while a supervised gold-aware state logistic comparator attains 0.02835. This check motivates the prospective GSM design below and is not used as its validation set.

\paragraph{Scoring sensitivity.} Alternative rationale-processing rules change the cohort's aggregate quantities modestly. Relative to the original rule, which cuts each rationale at its first curly brace and at 40 words, the 64-word rule raises independent-plurality accuracy from 0.3920 to 0.4155 and coverage by 3.8 points. A separate Ministral check on 50 items gives an accuracy difference of 2.0 points between corrected encoding and the cohort's encoding, with a 95\% item-bootstrap interval of $[-6.0,11.0]$. Cross-model comparisons involving Ministral in this cohort therefore retain this qualification.

\subsection{Native Generation}
\label{app:native-validation}

The native cohort uses Qwen3.5-2B and 4B on 1,200 new MMLU-Pro items and 200 new Word16 items per model, with one 16-call sequence per item. Responses are generated with native chat templates and parsed directly into a canonical answer or abstention. Word16 consists of sixteen-word sorting instances with exact reference answers that preserve multiplicity. An exact programmatic sorter exists, so this task is used as a controlled mechanism setting.

\begin{table}[htbp]
\centering\small\setlength{\tabcolsep}{4pt}
\caption{State-conditioned recovery in the native-generation cohort. Eligible prefixes are covered but unselected. Blocked means $L(v_{12})>4$, and late means first discovery at calls 9--16.}
\label{tab:native-recovery}
\begin{tabular}{llrrrr}
\toprule
Model & Task & \shortstack{Converts from\\call 8} & \shortstack{Blocked /\\eligible at 12} & Late feasible & Late converts\\
\midrule
2B & MMLU-Pro & $16.9\%$ (34/201) & $105/239$ & $50/82$ & $13/82$\\
4B & MMLU-Pro & $11.9\%$ (13/109) & $54/131$ & $16/33$ & $3/33$\\
2B & Word16 & $59.7\%$ (37/62) & $0/59$ & $20/21$ & $8/21$\\
4B & Word16 & $22.6\%$ (14/62) & $15/89$ & $17/21$ & $3/21$\\
\bottomrule
\end{tabular}
\end{table}

For native MMLU-Pro, late discoveries convert at 15.85\% for 2B and 9.09\% for 4B (\cref{tab:native-recovery}), showing that the absence of late conversion in the forced-choice cohort is protocol-specific. Early discovery also leaves many failures. Of the native MMLU-Pro failures, 187 of 256 (2B) and 109 of 139 (4B) are discovered by call 8. Word16 likewise contains both late recovery and blocked states, although its role is primarily mechanistic.

\subsection{Cross-Protocol Failure Composition}
\label{app:unified-failures}

\begin{table}[!htbp]
\centering\small\setlength{\tabcolsep}{4pt}
\caption{Endpoint coverage, accuracy, and unified failure counts. I: infeasible at discovery. F: feasible at discovery but never selected. L: selected and later lost. GSM-Symbolic includes the law-estimation sequences, and only 5,000 of its trajectories enter forecast evaluation. All rows use $B=16$. Public rows use trajectories from samples 2,001--10,000 of 127 GSM8K and 128 MATH problems. Invalid calls, as a share of all calls, are 0.0\% for the forced-choice protocol (five checkpoints and 27B), 11.2\% for GSM-Symbolic, 16.6\% and 7.5\% for native MMLU-Pro (2B, 4B), 9.2\% and 0.1\% for native Word16, 9.1\% for fresh MMLU-Pro, 0.1\% and 0.0\% for public GSM8K and 6.9\% and 6.8\% for public MATH (8B, 70B).}
\label{tab:unified-failures}
\label{tab:native-main}
\begin{tabular}{@{}lrrrrrr@{}}
\toprule
Cohort & Trajectories & Coverage & Accuracy & I & F & L\\
\midrule
Forced-choice MMLU-Pro & 2,000 & 69.20 & 41.55 & 74 & 287 & 192\\
GSM-Symbolic & 6,250 & 89.92 & 78.82 & 68 & 359 & 267\\
Native MMLU-Pro, 2B & 1,200 & 83.08 & 61.75 & 32 & 131 & 93\\
Native MMLU-Pro, 4B & 1,200 & 88.42 & 76.83 & 17 & 72 & 50\\
Native Word16, 2B & 200 & 92.50 & 70.00 & 1 & 27 & 17\\
Native Word16, 4B & 200 & 92.00 & 50.00 & 4 & 38 & 42\\
\midrule
Forced-choice MMLU-Pro, 27B & 400 & 84.50 & 75.50 & 7 & 15 & 14\\
Fresh MMLU-Pro, 4B (Base, $T=0.7$) & 768 & 86.85 & 74.74 & 10 & 51 & 32\\
Public GSM8K, Llama-3-8B & 63{,}500 & 97.95 & 86.04 & 576 & 3{,}551 & 3{,}435\\
Public GSM8K, Llama-3-70B & 63{,}500 & 99.21 & 96.81 & 93 & 689 & 739\\
Public MATH, Llama-3-8B & 64{,}000 & 66.40 & 38.75 & 1{,}676 & 9{,}798 & 6{,}220\\
Public MATH, Llama-3-70B & 64{,}000 & 80.00 & 58.63 & 1{,}383 & 7{,}147 & 5{,}143\\
\bottomrule
\end{tabular}
\end{table}

Across all protocols, coverage exceeds endpoint accuracy and covered failures include both failure to acquire the lead and failure to retain it. The relative frequencies differ substantially by protocol, which is consistent with the state-based framework: discovery alone does not determine endpoint selection.

\subsection{Larger Models and Budgets}
\label{app:scope-scale}

\paragraph{Data.}
Three replays reuse existing samples and make no model call. (i) Qwen3.5-27B was sampled under the forced-choice protocol of \cref{app:historical-protocol} for the judge study, with the same 200 items, prompt and forced-choice scoring and two 16-call sequences per item. (ii) The public Large Language Monkeys release \citep{brown2024largelanguagemonkeys} contains 10,000 samples per problem from Llama-3-8B-Instruct and Llama-3-70B-Instruct \citep{grattafiori2024llama3} on 127 GSM8K \citep{cobbe2021gsm8k} and 128 MATH problems \citep{hendrycks2021math}, generated at temperature 0.6 with few-shot prompts and a 512-token cap. GSM8K identities are those of \cref{app:lock-time-validation}. For MATH, the final answer is extracted and normalized exactly as by the dataset's own grader. An answer without a final-answer line, including a truncated one, abstains (6.9\% and 6.8\% of samples). A deterministic identity rule then merges forms that the grader treats as equal. It was developed from a gold-label audit of single samples and is applied without gold labels. It keys numeric values to 12 significant digits, so $\frac12$, $0.5$ and $\frac24$ coincide while $\frac13$ and $0.333$ stay distinct. It also identifies polynomial expressions in single-letter variables up to reordering and expansion ($12x-34$ and $-34+12x$), removes one pair of parentheses enclosing a whole answer without commas and removes the base subscript of a numeral ($52_8$ and $52$). Our correctness labels disagree with the dataset's on 0.07\% and 0.09\% of samples. Most disagreements arise because the grader compares a comma-separated list by its first element and accepts $24$ for the gold answer $24,-3$. We keep the stricter identity. The rule was fixed before any vote state was computed. Samples 2,001--10,000 form disjoint trajectories in stored order at each budget, and samples 1--2,000 are reserved for response laws. (iii) The 768-input fresh MMLU-Pro study of Qwen3.5-4B (\cref{app:temp-perm}) supplies four arms, Base or Permuted at temperature 0.7 or 1.0. Classes follow \cref{app:experimental-overview} with $16$ replaced by $B$, and late discovery means $t>B/2$.

\paragraph{Results.}
The discovery-to-decision gap appears in every replay (\cref{tab:scope-budget,fig:scope-gap-budget}). For 27B, 36 of its 338 covered trajectories fail. Of these, 7 are infeasible at discovery, 15 are feasible but never selected and 14 are selected then lost, and none of its eight late discoveries converts. On public data the gap is small for Llama-3-70B on GSM8K (2.4 points at $B=16$) and large on MATH. On MATH it grows with the budget, because coverage keeps rising while plurality accuracy approaches the empirical fraction of problems whose modal answer is correct. In a post hoc check over all 10,000 samples, which overlap the evaluation samples, the mode is correct on 43.0\% and 60.2\% of MATH problems, against plurality accuracies of 43.2\% and 60.4\% at $B=1{,}024$. This agrees with the modal ceiling of \citet{bay2026modalceiling} under a stable iid law. It is not a finite-budget bound. At every budget, at least 89.9\% of MATH failures were still recoverable when gold first appeared. At $B\geq256$, at least 98\% of them became unreachable before the final call, so at large budgets the failures are decided after discovery. Late discoveries still convert at $B=16$ (479 of 5,740 and 424 of 4,181 on MATH), and none does at $B\geq256$. On the fresh MMLU-Pro inputs the gap is 11.5 to 13.3 points in all four arms, so neither temperature 1.0 nor permutation closes it. On GSM-Symbolic, concatenating the stored sequences of the three checkpoints whose software environment did not change between waves (\cref{app:gsm-sensitivities}) gives gaps of 5.8, 6.3, 6.1 and 6.8 points at $B=16$, 32, 48 and 64. Per checkpoint the gap is flat or grows slowly, from 4.9 to 4.8, 4.3 to 6.0 and 8.3 to 9.6 points for 4B, Gemma and Ministral.

\paragraph{Sensitivity.}
Two checks bound the dependence on our constructions. Rebuilding the public trajectories from a fixed-seed random permutation of samples 2,001--10,000, or from blocks that start half a budget later, changes coverage, accuracy and the gap by at most 1.0 point at every budget. The largest shift, with paired problem-bootstrap intervals, is $-1.0$ $[-1.9,-0.1]$ points in the 8B MATH gap at $B=1{,}024$ under the permutation. Scoring MATH with the dataset's own labels while keeping our identities for vote counting raises coverage by 0.6 to 2.0 points and accuracy by at most 0.35 points. The gaps therefore grow slightly, by at most 2.0 points, to 50.5 and 36.7 at $B=1{,}024$. For 8B and 70B, 28 and 16 problems contain a sample-level disagreement, and none of them changes the correctness of its modal answer.

\begin{figure}[!htbp]
\centering
\includegraphics[width=\linewidth]{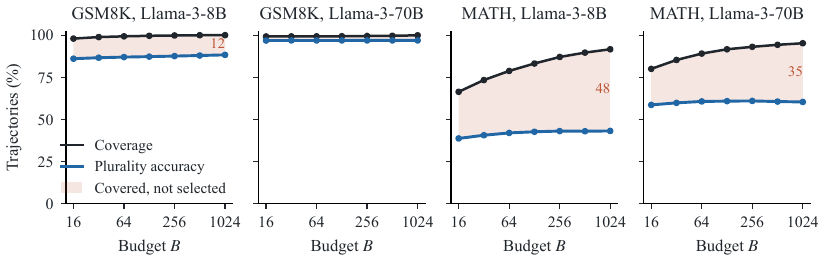}
\caption{Coverage and plurality accuracy against the budget on public samples from Llama-3-8B- and 70B-Instruct, for 127 GSM8K and 128 MATH problems. Trajectories are disjoint blocks of samples 2,001--10,000, from 63,500 or 64,000 trajectories at $B=16$ to 889 or 896 at $B=1{,}024$. The shaded band is the share of trajectories that sample the correct answer without selecting it, and the numbers give the gap at $B=1{,}024$ where it is at least 5 points (it is 3.0 points for 70B on GSM8K).}
\label{fig:scope-gap-budget}
\end{figure}

\begin{table}[!htbp]
\centering\footnotesize\setlength{\tabcolsep}{3.5pt}
\caption{The discovery-to-decision gap at larger budgets on public samples (127 GSM8K and 128 MATH problems). Gap is coverage minus plurality accuracy, in points with problem-cluster bootstrap 95\% intervals. Failures are covered failures as a share of covered trajectories. Recoverable is the share of covered failures still recoverable when gold first appeared. Closed is the share of covered failures whose correct answer became unreachable before the final call. Late converts counts late discoveries ($t>B/2$) selected at the endpoint, over all late discoveries.}
\label{tab:scope-budget}
\resizebox{\linewidth}{!}{%
\begin{tabular}{@{}llrrrlrrrr@{}}
\toprule
Task, model & $B$ & Traj. & Coverage & Accuracy & Gap [95\% CI] & Failures & Recoverable & Closed & Late converts\\
\midrule
GSM8K, 8B & 16 & 63{,}500 & 98.0 & 86.0 & 11.9 $[7.5,16.7]$ & 12.2 & 92.4 & 84.9 & 66/1{,}284\\
 & 64 & 15{,}875 & 99.3 & 87.0 & 12.3 $[7.3,17.8]$ & 12.4 & 98.3 & 95.6 & 0/75\\
 & 256 & 3{,}937 & 99.8 & 87.6 & 12.2 $[7.0,18.0]$ & 12.2 & 99.2 & 98.8 & 0/5\\
 & 1024 & 889 & 100.0 & 88.3 & 11.7 $[6.4,17.4]$ & 11.7 & 100.0 & 100.0 & 0/0\\
GSM8K, 70B & 16 & 63{,}500 & 99.2 & 96.8 & 2.4 $[0.2,5.1]$ & 2.4 & 93.9 & 93.3 & 1/108\\
 & 1024 & 889 & 99.9 & 96.9 & 3.0 $[0.7,6.2]$ & 3.0 & 85.2 & 100.0 & 0/4\\
MATH, 8B & 16 & 64{,}000 & 66.4 & 38.8 & 27.6 $[22.9,32.5]$ & 41.6 & 90.5 & 76.0 & 479/5{,}740\\
 & 64 & 16{,}000 & 78.8 & 42.0 & 36.8 $[30.0,43.7]$ & 46.7 & 96.3 & 95.3 & 2/876\\
 & 256 & 3{,}968 & 87.0 & 43.1 & 43.9 $[36.4,52.2]$ & 50.5 & 97.9 & 99.3 & 0/157\\
 & 1024 & 896 & 91.6 & 43.2 & 48.4 $[40.2,56.8]$ & 52.9 & 98.4 & 99.5 & 0/14\\
MATH, 70B & 16 & 64{,}000 & 80.0 & 58.6 & 21.4 $[16.5,26.5]$ & 26.7 & 89.9 & 78.0 & 424/4{,}181\\
 & 64 & 16{,}000 & 89.1 & 60.7 & 28.4 $[21.9,35.4]$ & 31.9 & 96.9 & 94.6 & 20/589\\
 & 256 & 3{,}968 & 93.1 & 61.0 & 32.1 $[24.8,39.9]$ & 34.5 & 98.4 & 98.2 & 0/59\\
 & 1024 & 896 & 95.2 & 60.4 & 34.8 $[27.1,43.0]$ & 36.6 & 99.4 & 99.4 & 0/6\\
\bottomrule
\end{tabular}%
}
\end{table}

\subsection{Closed Windows under a Judge Selector}
\label{app:selector-intervention}

Reachability and locking are defined for plurality over retained unit votes. This pre-specified study asks what a different selector does in states that plurality can no longer recover.

\paragraph{Design.}
Two stored 16-call sequences per item and source are cut at call 12, giving 2,000 states from 200 MMLU-Pro items and five sources. The 0.8B, 2B and Gemma pools reuse the forced-choice cohort of \cref{app:historical-protocol}. The 4B and Ministral pools use separately generated independent-sampling sequences on the same items with the same prompt and scoring. Gold-aware strata, enumerated before any judging, contain 268 vote-infeasible states ($L(v_{12})\geq5$) on 132 items, 248 covered-but-unselected feasible states, 819 currently correct states and 665 states without gold. The realized call-16 continuation selects gold in none of the 268 infeasible and in 38 of the 248 feasible states. Each smaller model judges all 400 states of its own pool, and Qwen3.5-27B judges all five pools as a secondary analysis. The judge sees the question, the distinct candidates of calls 1--12 in a fixed pool-seeded order, each with the rationale of its first occurrence, a note that every candidate may be wrong, and no counts. One letter is then scored over the question's full option set, so answers outside the pool are possible. No gold label or call after 12 enters the prompt. The comparator is the same model's direct-solving accuracy on the item over eight calls, with two further direct calls per item, generated alongside the judge calls, as a sensitivity. Estimates average states within an item and weight eligible items equally, and intervals resample all 200 items jointly (2,000 draws).

\begin{table}[!htbp]
\centering\footnotesize\setlength{\tabcolsep}{4pt}
\caption{Judge selection in the 268 vote-infeasible call-12 states. Rates (\%) and differences (points) use equal item weights with joint item-bootstrap 95\% intervals. Each smaller model judging its own pool is primary. Per-model rows are unadjusted, and 27B judging the same states is secondary.}
\label{tab:selector-primary}
\begin{tabular}{@{}lrrrl@{}}
\toprule
Judge & States & Judge & Direct & Difference [95\% CI]\\
\midrule
Own pool, pooled & 268 & 34.23 & 14.83 & $+19.40$ $[+12.39,+26.84]$\\
\quad Qwen3.5-0.8B & 70 & 27.68 & 10.27 & $+17.41$ $[+6.36,+30.05]$\\
\quad Qwen3.5-2B & 55 & 39.13 & 12.23 & $+26.90$ $[+13.33,+41.41]$\\
\quad Qwen3.5-4B & 40 & 27.42 & 21.37 & $+6.05$ $[-10.23,+23.57]$\\
\quad Gemma-4-E2B & 42 & 40.62 & 12.50 & $+28.12$ $[+12.50,+43.75]$\\
\quad Ministral-3-3B & 61 & 35.56 & 14.44 & $+21.11$ $[+7.77,+34.03]$\\
Qwen3.5-27B, all pools & 268 & 75.72 & 70.36 & $+5.36$ $[+0.01,+10.76]$\\
\bottomrule
\end{tabular}
\end{table}

\begin{table}[!htbp]
\centering\footnotesize\setlength{\tabcolsep}{4pt}
\caption{Judge selection in the other strata (descriptive). Differences are judge minus the stated comparator, in points, with equal item weights and joint item-bootstrap 95\% intervals. The call-13 and call-16 comparators use one and four further plurality calls. Judging adds rationale context and a scoring pass, so these comparisons are not cost-matched.}
\label{tab:selector-descriptive}
\begin{tabular}{@{}lrll@{}}
\toprule
Stratum: comparator & States & Own-pool judges & Qwen3.5-27B\\
\midrule
Covered, feasible: plurality at 16 & 248 & $+22.31$ $[+13.60,+30.88]$ & $+67.26$ $[+58.60,+74.83]$\\
Currently correct: plurality at 13 & 819 & $-26.18$ $[-31.43,-21.18]$ & $-13.93$ $[-19.42,-8.90]$\\
Gold absent: own direct solving & 665 & $+0.85$ $[-1.13,+3.18]$ & $-0.62$ $[-4.02,+2.73]$\\
All states: plurality at 16 & 2,000 & $+0.05$ $[-2.35,+2.50]$ & $+32.50$ $[+28.40,+37.05]$\\
\bottomrule
\end{tabular}
\end{table}

\paragraph{Results and scope.}
Four of the five own-pool contrasts are positive and the 4B contrast is inconclusive (\cref{tab:selector-primary}). The 27B lower bound is close to zero, and only its Ministral-pool contrast is positive among the five per-pool contrasts. Among the 819 currently correct states, own-pool judges lose gold in 29.17\% $[24.57,34.23]$ and 27B in 16.92\% $[11.99,22.17]$, and the own-pool all-state difference from four further votes is $+0.05$ points, so an overall improvement is not established (\cref{tab:selector-descriptive}). The large 27B all-state gain reflects a much stronger model at a different cost. The two further direct calls differ from the comparator by less than two points for every model except 4B ($-5.6$ $[-8.9,-2.5]$), and using them as the comparator gives $+20.7$ $[12.8,29.2]$ for the own-pool contrast. Candidates and rationale context jointly define the intervention, and gold-aware strata supply no routing rule. In the vocabulary of \cref{sec:structure}, a closed conversion window is closed for plurality: the judge reopens some windows and closes others that plurality had kept open.

\FloatBarrier

\section{Prospective Forecasting from the Vote State}
\label{app:gsm-forecasting}

\subsection{Prospective Design}

GSM-Symbolic \citep{mirzadeh2024gsmsymbolic} contributes 250 items from 100 GSM8K-derived template families. For each model-item pair, one independent 16-call sequence estimates the response law, while each forecast sequence supplies its own first eight calls as the realized state. Forecasts are computed and fixed before calls 9--16 are generated. Invalid or unparseable completions consume calls and contribute no vote. This design evaluates $Q_r(v;\widehat p)$ prospectively while separating the observed state from the suffix it predicts.

The first evaluation wave contains two forecast sequences per model-item pair and the second wave adds two more using the same items and response-law sequence. The primary diagnostic is the covered-but-unselected call-8 stratum, where endpoint recovery remains unresolved by the current plurality decision.

\subsection{Forecast Performance}

\begin{table}[!htbp]
\centering\small\setlength{\tabcolsep}{4pt}
\caption{GSM-Symbolic calibration in the covered-but-unselected call-8 stratum. Differences are observed minus forecast, in percentage points. The pooled row combines the two evaluation waves, and the range rows show pooled calibration by forecast magnitude.}
\label{tab:gsm-calibration}
\label{tab:gsm-calibration-waves}
\label{tab:gsm-waves}
\begin{tabular}{lrrrr}
\toprule
Population & $n$ & Observed & Forecast & Difference\\
\midrule
Wave 1 & 235 & 0.289 & 0.311 & $-2.1$\\
Wave 2 & 255 & 0.271 & 0.259 & $+1.2$\\
Pooled & 490 & 0.280 & 0.284 & $-0.4$\\
\midrule
Forecast $[0,0.05)$ & 189 & 0.063 & 0.008 & $+5.5$\\
Forecast $[0.05,0.2)$ & 83 & 0.169 & 0.121 & $+4.8$\\
Forecast $[0.2,0.5)$ & 93 & 0.333 & 0.350 & $-1.7$\\
Forecast $[0.5,1]$ & 125 & 0.640 & 0.759 & $-11.9$\\
\bottomrule
\end{tabular}
\end{table}

The pooled AUC within this stratum is 0.825 $[0.77,0.88]$. The mean calibration error is small, while the lowest and highest forecast ranges show opposite errors, indicating that short response-law estimation is most consequential near the extremes.

\begin{table}[!htbp]
\centering\small\setlength{\tabcolsep}{4pt}
\caption{Wave-1 Brier-score contrasts for the state-conditioned plug-in forecast. Differences are forecast minus comparator, so negative values favor the state-conditioned forecast.}
\label{tab:gsm-brier}
\begin{tabular}{llcc}
\toprule
Comparator & Population & Difference & 95\% interval\\
\midrule
Discovery-bin constant & covered, unselected & $-0.050$ & $[-0.078,-0.015]$\\
Law-only & all forecasts & $-0.023$ & $[-0.028,-0.019]$\\
Persistence & all forecasts & $-0.011$ & $[-0.017,-0.005]$\\
Law-only & covered, unselected & $-0.032$ & $[-0.051,-0.014]$\\
Persistence & covered, unselected & $-0.135$ & $[-0.193,-0.083]$\\
MMLU-Pro-fitted logistic & all forecasts & $+0.001$ & $[-0.001,+0.004]$\\
MMLU-Pro-fitted logistic & covered, unselected & $+0.006$ & $[-0.013,+0.024]$\\
\bottomrule
\end{tabular}
\end{table}

These comparisons show that conditioning on the realized vote state adds information beyond discovery timing, the response law alone, and persistence. Performance is similar to the supervised gold-aware state comparator used as a diagnostic.

\subsection{Scope and Sensitivity}
\label{app:gsm-sensitivities}

The second wave was added after the first-wave calibration was inconclusive, so the pooled estimate is descriptive. For two checkpoints, the second wave ran in a different software environment with the same weights, prompt tokens and decoding settings. In 55 of the 189 pooled lowest-range states, gold is absent from the 16-call law-estimation sequence and therefore receives zero plug-in mass. Seven of those states nevertheless select gold at the endpoint. The post-discovery feasibility convention of \cref{app:experimental-overview} differs from a pre-discovery convention on 24 late boundary discoveries, with no change in the observed endpoint conversions. Short-sample law estimation and the iid assumption both limit calibration, and this design does not separate their contributions. A secondary forecaster, selected among five candidates on the first wave, was pre-specified for the second. It adds the state's own eight prefix calls to the law evidence and shrinks the gold share toward a per-model empirical-Bayes prior fitted on the law sequences. On the second wave's covered-but-unselected states, its Brier score differs from the plug-in's by $-0.007$ $[-0.019,+0.005]$ and its calibration error is $+2.1$ points (90\% interval $[-1.5,+5.7]$). Both intervals include zero.

\subsection{Longer Law Samples}
\label{app:scope-forecast}

Each GSM-Symbolic law comes from only 16 calls. Two retrospective replays therefore vary the law sample and keep the exact forecast of \cref{eq:abstention-recursion} unchanged. On GSM-Symbolic, each call-8 state is forecast again from the 64 calls of the other four sequences of the same model and item. On the public samples, each call-8 state of a $B=16$ trajectory from samples 2,001--10,000 is forecast from the first 16, 64, 256 or 2,000 samples of the problem. These law samples are disjoint from the evaluation samples and come from the same problems. In the covered-but-unselected stratum, longer laws reduce the extreme-range errors and raise the AUC (\cref{tab:scope-forecast,fig:scope-forecast-lawsize}). With 2,000-sample laws, every forecast range with at least 100 states has a point discrepancy of at most 2.1 points, and the 19-state middle range of 70B GSM8K keeps $-19.4$. Across all 63,500 or 64,000 states, the mean point discrepancy is at most 0.1 point. On GSM-Symbolic the gains are uneven. The 64-call law raises the AUC and lowers the Brier score for four of the five checkpoints, with Gemma-4-E2B the exception. Per-checkpoint strata contain only 36 to 205 states, and the pooled mean and middle-range errors grow. These replays indicate that law-estimation uncertainty contributes to the miscalibration. The GSM-Symbolic 64-call laws use sequences generated after the forecasts were fixed, and the public samples come from one fixed prompt per problem.

\begin{table}[!htbp]
\centering\footnotesize\setlength{\tabcolsep}{3pt}
\caption{Forecast calibration in the covered-but-unselected call-8 stratum by law sample size. Differences are observed minus forecast, in points, with cluster-bootstrap 95\% intervals over template families or problems. The last range is $[0.5,1]$, shown with its number of states. Absent counts the states whose correct answer never occurs in the law sample.}
\label{tab:scope-forecast}
\resizebox{\linewidth}{!}{%
\begin{tabular}{@{}llrrrlrrrr@{}}
\toprule
Data & Law & States & Observed & Forecast & Difference & $[0,0.05)$ & $[0.5,1]$ (n) & AUC & Absent\\
\midrule
GSM-Symbolic & 16 calls (prospective) & 490 & 0.280 & 0.284 & $-0.4$ $[-3.6,+3.1]$ & $+5.5$ & $-11.9$ (125) & 0.825 & 55\\
 & 64 calls & 490 & 0.280 & 0.293 & $-1.3$ $[-3.9,+1.4]$ & $+3.0$ & $-3.4$ (128) & 0.871 & 11\\
GSM8K, 8B & 16 & 6{,}841 & 0.215 & 0.245 & $-3.0$ $[-7.8,+1.6]$ & $+0.7$ & $-13.5$ (1{,}492) & 0.865 & 141\\
 & 2{,}000 & 6{,}841 & 0.215 & 0.219 & $-0.3$ $[-1.2,+0.6]$ & $-0.3$ & $+0.6$ (1{,}145) & 0.907 & 0\\
GSM8K, 70B & 16 & 1{,}543 & 0.180 & 0.177 & $+0.3$ $[-3.6,+6.4]$ & $+1.4$ & $+5.8$ (280) & 0.948 & 4\\
 & 2{,}000 & 1{,}543 & 0.180 & 0.179 & $0.0$ $[-1.6,+1.3]$ & $+0.1$ & $+1.0$ (293) & 0.970 & 0\\
MATH, 8B & 16 & 13{,}730 & 0.229 & 0.195 & $+3.5$ $[-0.3,+6.7]$ & $+3.7$ & $-12.3$ (1{,}816) & 0.806 & 1{,}781\\
 & 2{,}000 & 13{,}730 & 0.229 & 0.229 & $0.0$ $[-0.6,+0.6]$ & $-0.2$ & $-1.5$ (2{,}323) & 0.875 & 4\\
MATH, 70B & 16 & 10{,}574 & 0.206 & 0.234 & $-2.8$ $[-8.8,+2.7]$ & $+3.4$ & $-27.9$ (1{,}933) & 0.795 & 1{,}238\\
 & 2{,}000 & 10{,}574 & 0.206 & 0.207 & $-0.1$ $[-0.7,+0.6]$ & $+0.1$ & $-1.0$ (1{,}519) & 0.883 & 2\\
\bottomrule
\end{tabular}%
}
\end{table}

\begin{figure}[!htbp]
\centering
\includegraphics[width=\linewidth]{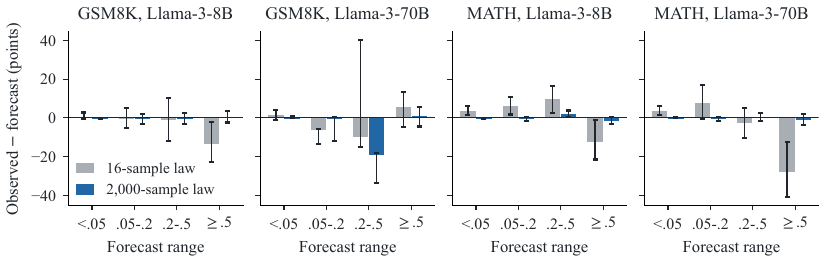}
\caption{Observed minus forecast endpoint rate by forecast range in the covered-but-unselected call-8 stratum of public $B=16$ trajectories, with laws from 16 (gray) or 2,000 (blue) disjoint samples of the same problem. Bars are problem-cluster bootstrap 95\% intervals. The 70B GSM8K range $[0.2,0.5)$ contains 19 states.}
\label{fig:scope-forecast-lawsize}
\end{figure}

\subsection{Forecasts of an Intervention's Effect}
\label{app:state-forecast}

This study tests the state dependence of \cref{app:state-competition} prospectively: from the same realized prefix, does the exact endpoint calculation predict how an intervention changes the fixed-budget outcome?

\paragraph{Design.}
The study uses 128 MMLU-Pro inputs selected by category round-robin and a fixed pseudo-random ranking, without using model outcomes and excluding every input of the other MMLU-Pro cohorts. Eligible questions have three to ten options and pass a case-aware lexical guard against references to option positions that retains ordinary uses of the article ``a''. Lexical eligibility does not prove semantic invariance to option order. Qwen3.5-4B uses the native settings of \cref{app:checkpoints}. Base keeps the option order, while Permuted shuffles the options independently on each call and maps the parsed letter back to its identity. Invalid and truncated outputs abstain and consume a call, and both arms share a fixed item-specific tie priority. For each input, twelve Base responses give realized prefixes at $t=8$ and $t=12$ with $B=16$, and an independent reference pool of 256 responses per arm estimates each arm's full answer-and-abstention law by its Dirichlet$(1/2)$ posterior mean. A simulation on specified laws, run before any model outcome, set this reference size.

Exact coefficient extraction (\cref{app:endpoint-laws}) gives each arm's endpoint probability from each state, and the forecast effect is their Permuted-minus-Base difference. All forecasts were fixed before any validation call. Validation then generated 32 independent eight-call continuations per arm and input, and the first four calls of each also serve the $t=12$ state. Two controls keep each arm's correct and abstention masses. One spreads the wrong mass uniformly over wrong identities. The other averages over 1,024 fixed-seed permutations of the wrong-mass vector across wrong identities, which preserves its spectrum and collision but not its alignment with the current votes. For control $b$, the primary score of input $i$ averages $(\widehat\Delta_{i,t}-f_{i,t,b})^2-(\widehat\Delta_{i,t}-f_{i,t,\mathrm{full}})^2$ over its two states, where $\widehat\Delta_{i,t}$ is the observed difference in mean continuation correctness. The observation-noise term cancels in this paired difference. Superiority requires the lower end of a 97.5\% whole-input bootstrap interval (20,000 draws, Bonferroni over the two controls) to exceed zero. Gold defines the success event and the wrong identities, so the study is an offline diagnostic, and reference acquisition lies outside the 16-call budget.

\begin{table}[!htbp]
\centering\footnotesize\setlength{\tabcolsep}{4pt}
\caption{Effect forecasts on 128 inputs and 256 states. The squared-error gain is control minus full loss (positive favors the full forecast) with the pre-specified 97.5\% intervals. Brier scores over individual continuation outcomes and the squared error to the noisy observed effect are secondary, with 95\% intervals.}
\label{tab:state-forecast-primary}
\begin{tabular}{@{}llll@{}}
\toprule
Forecast & Squared-error gain & Endpoint Brier & Effect squared error\\
\midrule
Full state-based & --- & 0.0146 $[0.0079,0.0222]$ & 0.00092 $[0.00032,0.00170]$\\
Uniform wrong mass & 0.0058 $[0.0010,0.0134]$ & 0.0212 $[0.0109,0.0334]$ & 0.00670 $[0.00198,0.01330]$\\
Wrong-identity shuffle & 0.0030 $[0.0004,0.0068]$ & 0.0189 $[0.0100,0.0293]$ & 0.00390 $[0.00134,0.00730]$\\
\bottomrule
\end{tabular}
\end{table}

\begin{table}[!htbp]
\centering\footnotesize\setlength{\tabcolsep}{4pt}
\caption{Effect-forecast bins, fixed before validation, with states weighted equally and item-bootstrap 95\% intervals. An input can contribute to two bins.}
\label{tab:state-calibration-bins}
\begin{tabular}{@{}lrll@{}}
\toprule
Forecast bin & Inputs / states & Mean forecast & Mean observed effect\\
\midrule
$[-1,-0.1)$ & 4 / 6 & $-0.349$ $[-0.545,-0.165]$ & $-0.391$ $[-0.585,-0.219]$\\
$[-0.1,-0.02)$ & 9 / 11 & $-0.053$ $[-0.066,-0.040]$ & $-0.060$ $[-0.096,-0.028]$\\
$[-0.02,0.02)$ & 117 / 224 & $0.000$ $[-0.001,0.000]$ & $0.000$ $[-0.001,0.001]$\\
$[0.02,0.1)$ & 3 / 3 & $0.037$ $[0.029,0.051]$ & $0.021$ $[0.000,0.031]$\\
$[0.1,1]$ & 9 / 12 & $0.370$ $[0.221,0.495]$ & $0.362$ $[0.250,0.482]$\\
\bottomrule
\end{tabular}
\end{table}

\paragraph{Results.}
Both primary comparisons favor the full forecast (\cref{tab:state-forecast-primary}). Of the 256 states, 129 have endpoints fixed by the prefix and remaining budget (98 guaranteed correct, 31 guaranteed incorrect). The 127 other states, from 109 inputs, carry the contrast, with gains of $0.0117$ $[0.0027,0.0241]$ and $0.0060$ $[0.0013,0.0124]$ over the two controls. Both prefixes favor the full forecast ($0.0090$ and $0.0024$ at $t=8$, $0.0025$ and $0.0035$ at $t=12$). By category, the gains come mainly from history, law, philosophy and other. In most remaining categories the three forecasts nearly coincide, and no category favors a control by more than $0.002$. The five pre-specified bins are close to calibrated (\cref{tab:state-calibration-bins}). Among the 32 states from 21 inputs whose forecast effect exceeds $0.02$ in magnitude, 29 agree with the observed sign and three show no change. Validation outputs contain 3,501 invalid responses, almost all truncations at the token cap, each counted as an abstention.

\paragraph{Oracle recoding of two wrong identities.}
Natural permutation changes the whole response law. To test the transfer of \cref{prop:pivotal-transfer} directly, we recoded the independent Base continuations at each forecast state. Let $a$ be the wrong option with the most prefix votes and $b$ the one with the fewest (ties broken by reference mass and priority, fixed before validation). A fixed-seed coin relabels each $a$ response as $b$ ($a\to b$), or each $b$ response as $a$ ($b\to a$), with probability $\alpha\in\{0.25,0.5\}$, leaving every correct and abstaining response unchanged. Moving mass from the most- to the least-voted wrong answer raises the endpoint rate and the reverse lowers it, with average predicted and observed effects close (\cref{tab:state-forecast-oracle}). State-level direction checks are sparse. For a transferred mass $\epsilon=\alpha p_a$ or $\alpha p_b$, with $p$ the Base reference law, the linear prediction $\epsilon\,dQ_r/d\epsilon$ from \cref{eq:wrong-transfer-derivative} differs from the exact change by at most $r(r-1)\epsilon^2$, because $Q_r$ is a polynomial in $\epsilon$ whose second derivative is at most $2r(r-1)$ in absolute value. A state is eligible when the linear prediction plus or minus this bound excludes zero. One state is eligible for $a\to b$ at each $\alpha$, and 18 and 9 for $b\to a$ at $\alpha=0.25$ and $0.5$. Most observed effects are zero and none has the opposite sign. The recoding uses gold and post-processes responses, so it tests only the transfer calculation and leaves the mechanism of natural permutation open.

\begin{table}[!htbp]
\centering\footnotesize\setlength{\tabcolsep}{5pt}
\caption{Oracle recoding of two wrong identities on the independent Base continuations: pre-computed finite-difference predictions and observed endpoint changes, in points, with item-bootstrap 95\% intervals.}
\label{tab:state-forecast-oracle}
\begin{tabular}{@{}lrrl@{}}
\toprule
Direction & $\alpha$ & Predicted & Observed\\
\midrule
$a\to b$ & 0.25 & $+0.47$ & $+0.39$ $[+0.16,+0.67]$\\
$a\to b$ & 0.5 & $+1.03$ & $+1.06$ $[+0.46,+1.82]$\\
$b\to a$ & 0.25 & $-0.06$ & $-0.05$ $[-0.11,0.00]$\\
$b\to a$ & 0.5 & $-0.12$ & $-0.07$ $[-0.16,-0.01]$\\
\bottomrule
\end{tabular}
\end{table}

\FloatBarrier

\section{Representation Diversity and Error Competition}
\label{app:future-protocols}

\subsection{Design and Primary Test}
\label{app:adaptive-sampling}
\label{app:concentration-results}

The diversification study uses separate development and held-out test inputs, with 64 development and 256 test items for Word16 and MMLU-Pro. The principal intervention independently permutes the input representation on each call: Word16 shuffles the sixteen-word multiset, while MMLU-Pro permutes option positions and maps the returned letter back to the canonical option identity. Base and Permuted use 16 calls at temperature 0.7, and the 4B Word16 study additionally includes a high-temperature arm. Development data choose the fixed comparator used for the adaptive-policy comparison, and the held-out Word16 test set supplies the three pre-specified primary decisions.

\paragraph{Sampling policies.}
All policies use a 16-call budget. Base uses the original input on all calls, whereas Permuted independently applies the task-specific semantics-preserving permutation on each call. High temperature uses the Base representation with temperature $1.0$. Probe observes eight Base calls followed by two Permuted probe calls. If either probe is invalid or disagrees with the Base plurality, the final six calls use Permuted inputs, and otherwise they use Base. All issued calls count toward the budget, canonical plurality uses all sampled responses selected by the policy, and routing uses only responses available at the decision point without gold labels.

\begin{table}[!htbp]
\centering
\caption{Pre-specified Word16 primary comparisons on 256 paired test items. Formal decisions use Holm-adjusted exact paired tests over the three-comparison family.}
\label{tab:primary-test-details}
\resizebox{\linewidth}{!}{%
\begin{tabular}{@{}lrrrrrl@{}}
\toprule
Comparison & Effect & 95\% CI & Discordant $+/-$ & 98.33\% CI & Holm-adjusted $p$ & Decision\\
\midrule
4B Base minus 2B Base & $-8.20$ & $[-16.02,-0.39]$ & $43/64$ & $[-17.58,1.17]$ & $0.105$ & inconclusive\\
4B Permuted minus Base & $+21.09$ & $[14.06,28.12]$ & $74/20$ & $[12.50,29.69]$ & $5.54\times10^{-8}$ & supported\\
4B Probe minus High temp. & $-3.52$ & $[-9.77,2.73]$ & $31/40$ & $[-11.33,4.30]$ & $0.342$ & inconclusive\\
\bottomrule
\end{tabular}%
}
\end{table}

The supported primary result is the 4B Word16 permutation intervention: raw-plurality accuracy rises from 57.42\% to 78.52\%, while single-call correctness changes by $-0.22$ points $[-2.71,2.17]$. The two remaining primary comparisons are inconclusive after multiplicity correction. The model comparison was pre-specified as a replication of an exploratory result on the 200 native Word16 items of \cref{app:native-validation}: there, in an analysis specified after the cohort had been examined, 4B minus 2B is $+11.47$ $[8.47,14.53]$ points in single-call correctness and $-20.0$ $[-29.5,-11.0]$ points in plurality accuracy, and the mean largest wrong count is 4.2 against 1.8. On the held-out test items, 4B has $11.96$ $[9.15,14.79]$ points higher single-call correctness, a coverage difference of $+1.95$ $[-1.95,5.86]$ points, a largest wrong count $2.20$ $[1.89,2.51]$ votes higher and conditional wrong collision $0.158$ $[0.136,0.180]$ higher, while the plurality difference is $-8.20$ points (exact paired $p=0.053$ before and $0.105$ after Holm correction). Under the pre-specified decision rule, no direction is stated for this comparison.

\subsection{Mechanism and Selector Dependence}
\label{app:diversification-controls}

For valid wrong-answer counts $w_a$ and $W=\sum_a w_a$, define $\widehat H_{\mathrm{wrong}}=\sum_a w_a(w_a-1)/(W(W-1))$ and $\widehat C_{\mathrm{all}}=\sum_a w_a(w_a-1)/(16\cdot15)$. On 4B Word16, permutation lowers the mean largest wrong count from 4.03 to 2.84, increases the number of distinct wrong identities by 1.82 $[1.51,2.13]$, and lowers the all-item wrong-collision statistic by 0.0535 $[0.0408,0.0673]$ (\cref{fig:concentration-empirical}). These changes are consistent with weaker repeated wrong rivals. Permutation also changes other features of the response law, so collision need not be the only mediator.

\begin{figure}[!htbp]
\centering
\includegraphics[width=0.78\linewidth]{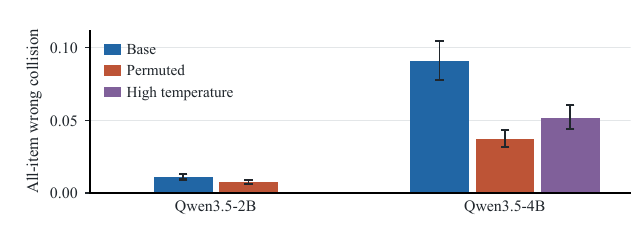}
\caption{All-item wrong collision on the 256 Word16 test items in each generation arm. Permutation substantially reduces repeated wrong identities while leaving single-call correctness nearly unchanged in the 4B comparison.}
\label{fig:concentration-empirical}
\end{figure}

Word16 also permits an input-defined multiset-validity filter and Borda fusion as stronger endpoint selectors. These controls substantially attenuate the representation effect, as shown in \cref{tab:selector-controls}.

\begin{table}[!htbp]
\centering\small
\caption{Selector dependence on the same 256 4B Word16 test items (accuracy, \%).}
\label{tab:selector-controls}
\begin{tabular}{lrrr}
\toprule
Generation & Raw plurality & Multiset-filtered & Borda fusion\\
\midrule
Base & 57.42 & 89.45 & 82.81\\
Permuted & 78.52 & 95.31 & 85.55\\
High temp. & 74.22 & 94.53 & 90.23\\
\bottomrule
\end{tabular}
\end{table}

\paragraph{Generator--selector interaction.} The Permuted-minus-Base gain is 21.09 points under raw plurality, 5.86 points under multiset-filtered plurality, and 2.73 points under Borda. The corresponding filtered-minus-raw interaction is $-15.23$ points with an approximate 98.33\% interval of $[-23.44,-7.03]$, and the Borda-minus-raw interaction is $-18.36$ $[-27.73,-8.98]$. The representation effect therefore depends strongly on the endpoint selector in this cohort.

\paragraph{A second model.}
The Word16 study also sampled Qwen3.5-2B under Base and Permuted generation on the same 256 test items, as a secondary descriptive comparison. The 2B pools are far less concentrated. Their all-item wrong collision is 0.011 against 0.091 for 4B Base, and their mean largest wrong count is 1.83 against 4.03. Permutation raises 2B raw-plurality accuracy from 65.62\% to 72.27\%, a gain of $+6.64$ $[0.00,13.28]$ points (46 and 29 discordant items, unadjusted exact $p=0.064$). It lowers the 2B all-item wrong collision by 0.0038 $[0.0017,0.0060]$ and changes single-call correctness by $+1.29$ $[-0.51,3.08]$ points. The 4B gain exceeds the 2B gain by $14.45$ $[5.47,23.44]$ points (paired item bootstrap). The larger gain accompanies greater baseline wrong-answer concentration, an association observed on only two checkpoints.

\subsection{Transfer to MMLU-Pro}
\label{app:diversification-secondary}

On the input-filtered 4B MMLU-Pro test population, permutation raises raw-plurality accuracy from 75.39\% to 81.64\%, a gain of 6.25 points $[3.12,9.77]$, while single-call correctness changes by 0.02 points $[-1.56,1.61]$. The lexical guard excludes 3,280 of 9,495 candidate inputs, including some ordinary uses of the article ``a'', so this transfer result applies to the filtered population used by the study.

\subsection{A Fresh Native Cohort}
\label{app:fresh-mechanism}

\paragraph{Design.}
A second pre-specified study uses separate sets of 96 development, 384 test and 48 online inputs for each of three tasks, disjoint from every earlier cohort: Word16 (every fifth input repeats a word), MMLU-Pro under a guard that excludes explicit option-label and position references but retains ordinary uses of the article ``a'' (663 of 9,171 candidate inputs excluded), and a directed shortest-path task with 16 vertices, 48 weighted edges and a unique optimal path of three to seven edges. Each input receives sixteen Qwen3.5-4B responses under Base, Permuted (shuffled words, options or edge list) and High temperature ($1.0$) generation, with the native settings of \cref{app:checkpoints}. Raw plurality counts every parsed answer. The necessary-condition selector checks the exact input multiset for Word16, and correct endpoints, known vertices, simplicity and existing edges for graphs, without checking order or optimality. Rejected responses consume calls, and an empty filtered pool falls back to raw plurality only at the cap. The two primary tests compare filtered with raw plurality on the Permuted pools of Word16 and graphs (exact McNemar tests, Holm correction). The Word16 selectors, including occurrence-based Borda fusion, and the three Permuted-minus-Base interactions were fixed before the test outcomes were opened, while MMLU-Pro contrasts are descriptive. Development data selected a fixed arm per task (High temperature for Word16, Permuted for MMLU-Pro, Base for graphs) for the online study in \cref{app:physical-execution}.

\begin{table}[!htbp]
\centering\footnotesize\setlength{\tabcolsep}{3.5pt}
\caption{Fresh native cohort, 384 test inputs per task and sixteen calls per arm (percentages). Coverage is unchanged by filtering. Invalid calls include truncation and parse failures, and every call consumes budget. Borda fusion applies to Word16 only.}
\label{tab:fresh-native-accuracy}
\begin{tabular}{@{}llrrrrrr@{}}
\toprule
Task & Generation & Raw & Filtered & Borda & Single call & Coverage & Invalid\\
\midrule
Word16 & Base & 59.64 & 89.06 & 82.55 & 33.45 & 95.31 & 0.07\\
 & Permuted & 74.22 & 92.19 & 85.42 & 32.28 & 98.70 & 0.10\\
 & High temp. & 76.56 & 92.45 & 85.94 & 36.49 & 97.92 & 0.10\\
MMLU-Pro & Base & 73.70 & 73.70 & --- & 68.72 & 84.90 & 7.24\\
 & Permuted & 76.30 & 76.30 & --- & 69.51 & 89.06 & 6.43\\
 & High temp. & 73.96 & 73.96 & --- & 68.46 & 86.46 & 7.47\\
Shortest path & Base & 92.45 & 93.49 & --- & 42.99 & 94.27 & 55.60\\
 & Permuted & 92.19 & 92.45 & --- & 41.23 & 93.23 & 57.19\\
 & High temp. & 85.16 & 85.68 & --- & 33.63 & 86.20 & 65.07\\
\bottomrule
\end{tabular}
\end{table}

\paragraph{Filtering.}
On the Permuted pools, filtering raises Word16 accuracy by $17.97$ points $[14.06,21.88]$, correcting 69 inputs and damaging none (Holm-adjusted exact $p=6.78\times10^{-21}$). On graphs it changes one decision ($+0.26$ $[0.00,0.78]$, adjusted $p=1$), because most incorrect calls are invalid and few parse to a wrong path. For MMLU-Pro, parsing already enforces the option alphabet, so filtering changes nothing. On Word16, filtering keeps 604 of 4,155 parsed wrong responses and every correct one: conditional wrong collision rises by $0.245$ $[0.189,0.302]$ on the 155 inputs eligible under both selectors, while all-call collision falls by $0.031$ $[0.027,0.034]$ and the largest wrong count by $1.72$ $[1.57,1.88]$ votes, the empirical counterpart of the example in \cref{app:necessary-filter}. At call 12, filtering lowers the mean recovery threshold from 1.35 to 0.21 (Base), from 0.65 to 0.15 (Permuted) and from 0.73 to 0.14 (High temperature), and every input whose correct answer was covered but infeasible (18, 2 and 10 inputs) becomes feasible again (\cref{fig:filter-recoverability}).

\begin{figure}[!htbp]
\centering
\includegraphics[width=0.9\linewidth]{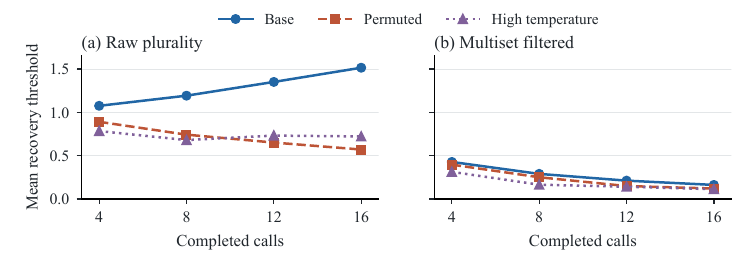}
\caption{Mean recovery threshold of the correct answer on the 384 fresh Word16 inputs after 4, 8, 12 and 16 calls (cap 16), under raw plurality (a) and multiset filtering (b), for each generation arm. The thresholds use reference answers offline.}
\label{fig:filter-recoverability}
\end{figure}

\paragraph{Error persistence.}
Take the most frequent wrong identity among the first eight Base calls and compare its frequency in the remaining eight Base calls with its frequency in the first eight Permuted calls: $0.147$ and $0.069$ on Word16 ($-0.078$ $[-0.095,-0.061]$, 377 inputs), $0.428$ and $0.350$ on MMLU-Pro ($-0.078$ $[-0.116,-0.042]$, 160 inputs) and $0.045$ and $0.004$ on graphs ($-0.042$ $[-0.069,-0.017]$, 36 inputs). Permutation makes a given mistake less likely to recur. On MMLU-Pro it also lowers conditional wrong collision by $0.089$ $[0.036,0.142]$ (134 inputs) and changes single-call correctness by $+0.8$ $[-0.6,2.1]$ points, coverage by $+4.2$ $[1.8,6.8]$ and plurality accuracy by $+2.6$ $[-0.3,5.5]$ (20 inputs corrected, 10 damaged).

\paragraph{Selector interaction.}
On Word16, the Permuted-minus-Base gain is $+14.58$ points $[8.59,20.32]$ under raw plurality, $+3.13$ $[-0.52,6.77]$ with filtering and $+2.86$ $[-1.82,7.55]$ with Borda fusion (descriptive 95\% intervals). The pre-specified interactions, with Bonferroni 98.33\% whole-input bootstrap intervals (20,000 draws), are $-11.46$ $[-17.97,-4.95]$ for filtered minus raw, $-11.72$ $[-19.53,-3.91]$ for Borda minus raw and $+0.26$ $[-4.69,5.47]$ for filtered minus Borda: filtering and Borda both attenuate the gain, and they are not distinguished from each other. High temperature gives a raw gain of $+16.93$ $[11.72,22.14]$ but also raises single-call correctness (\cref{tab:fresh-native-accuracy}). In descriptive duplicate-word strata, permutation's raw gain exceeds high temperature's on the 77 inputs with a repeated word ($+24.7$ against $+13.0$), and the order reverses on the 307 without ($+12.1$ against $+17.9$).

\subsection{Temperature and Permutation on a Fresh MMLU-Pro Hold-out}
\label{app:temp-perm}

\paragraph{Design.}
The earlier MMLU-Pro permutation effects (\cref{app:diversification-secondary,app:fresh-mechanism}) were measured at temperature $0.7$. This pre-specified study asks whether a permutation effect appears at $0.7$ on new inputs, whether one appears at $1.0$, and whether the two differ. It uses 768 new test inputs and 32 development inputs, drawn by category round-robin and a fixed pseudo-random ranking from MMLU-Pro inputs with three to ten options. None of these inputs appears in any other cohort of this paper. A case-aware lexical guard of the same form as in \cref{app:fresh-mechanism} is applied. Twelve categories contain 55 test inputs and two contain 54, and estimates weight categories by their test counts.

Each input receives sixteen Qwen3.5-4B responses in each cell of Base or Permuted crossed with temperature $0.7$ or $1.0$ (native nonthinking template, BF16, top-$p=0.9$, top-$k=20$, repetition penalty $1.05$, 4,096-token cap). Permuted shuffles the options independently on each call and maps the parsed letter back to its original identity. Each call index of an input uses the same permutation at both temperatures. Plurality uses the item-specific tie priority of the earlier MMLU-Pro studies. Invalid, truncated and failed calls abstain and consume their call. A call that reaches the cap abstains even when an answer appears before the cutoff. Three pre-specified contrasts share one Holm correction. They are the permutation effect at $0.7$, the permutation effect at $1.0$, both tested with exact McNemar tests, and their interaction, the $1.0$ effect minus the $0.7$ effect, tested with a category-stratified studentized bootstrap of whole four-cell input bundles (10,000 replicates). A contrast is called positive or negative when its Holm-adjusted $p<0.05$ and its 95\% interval excludes zero in that direction. It is called practically equivalent, a descriptive and unadjusted classification, when its 90\% interval lies within $\pm2$ points. A sensitivity analysis, added before any outcome was read, excludes the inputs without generated outputs.

\paragraph{Generation.}
Calls are generated in stages, and the test outputs are samples from this procedure. Each call first generates up to 1,024 tokens. Calls without an end token continue from the prompt and their own partial output, up to 2,048, 3,072 and 4,096 tokens. On 96 development token paths, the per-step sampling distributions of staged and single-pass generation differ by a mean total variation of $0.006$ and a 99th percentile of $0.06$--$0.08$. The pre-specified thresholds were $0.01$ for the mean and $0.02$ for the 99th percentile, so the tail threshold is exceeded. In a post hoc comparison, changing only the padding or the batch size of single-pass generation gives deviations of the same size (mean $0.005$--$0.008$, 99th percentile $0.06$--$0.08$). Per-step comparisons do not bound differences in whole sequences or accuracy. Eight inputs produced no outputs. Their calls count as abstentions, so these inputs are incorrect in all four cells of the primary analysis. Of the 12,288 calls per arm, 90.4--91.0\% gave a valid letter and 8.0--8.6\% reached the 4,096-token cap, mostly through long, non-repeating deliberation. Generation failures, all from the eight inputs without outputs, account for 1.04\%, and at most 0.03\% had an invalid letter or no final box (\cref{tab:temp-perm-status}).

\begin{table}[!htbp]
\centering\footnotesize\setlength{\tabcolsep}{4pt}
\caption{Status of the 12,288 calls per arm (percent). Truncated calls reached the 4,096-token cap and abstain.}
\label{tab:temp-perm-status}
\begin{tabular}{lrrrrr}
\toprule
Arm & Valid & Truncated & Generation failure & Invalid letter & No final box\\
\midrule
Base, $T=0.7$ & 90.87 & 8.06 & 1.04 & 0.02 & 0.01\\
Permuted, $T=0.7$ & 90.97 & 7.97 & 1.04 & 0.01 & 0.01\\
Base, $T=1.0$ & 90.35 & 8.58 & 1.04 & 0.00 & 0.03\\
Permuted, $T=1.0$ & 90.37 & 8.55 & 1.04 & 0.00 & 0.03\\
\bottomrule
\end{tabular}
\end{table}

\paragraph{Results.}
\Cref{tab:temp-perm} reports the four cells and the three contrasts. At $0.7$ the permutation effect is $+0.26$ points (95\% interval $[-1.56,2.08]$, 90\% interval $[-1.30,1.69]$, exact McNemar $p=0.89$), which meets the descriptive, unadjusted equivalence criterion. At $1.0$ it is $-0.52$ ($[-2.47,1.43]$, 90\% $[-2.08,1.04]$, $p=0.69$), which is inconclusive. The interaction is $-0.78$ ($[-3.09,1.52]$, bootstrap $p=0.52$), also inconclusive. All Holm-adjusted $p$-values equal $1$ and no contrast is positive or negative, so no accuracy benefit of permutation is detected at either temperature. Excluding the eight inputs without outputs raises every accuracy by $0.78$--$0.79$ points, moves no contrast by more than $0.02$ points and changes no classification. Moving from temperature $0.7$ to $1.0$ changes accuracy by $-0.26$ points for Base and $-1.04$ for Permuted.

\begin{table}[!htbp]
\centering\footnotesize\setlength{\tabcolsep}{4pt}
\caption{Fresh MMLU-Pro hold-out (768 inputs, Qwen3.5-4B, sixteen calls per cell). Accuracy is 16-call plurality correctness. Intervals are category-stratified item-bootstrap intervals, studentized for the interaction. Practical equivalence is a descriptive, unadjusted classification, and all Holm-adjusted $p$-values equal $1$.}
\label{tab:temp-perm}
\begin{tabular}{lrrrrrl}
\toprule
 & \multicolumn{2}{c}{Accuracy (\%)} & & & & \\
Contrast & Base & Permuted & Effect (points) & 95\% interval & 90\% interval & Classification\\
\midrule
At $T=0.7$ & 74.74 & 75.00 & $+0.26$ & $[-1.56,2.08]$ & $[-1.30,1.69]$ & practically equivalent\\
At $T=1.0$ & 74.48 & 73.96 & $-0.52$ & $[-2.47,1.43]$ & $[-2.08,1.04]$ & inconclusive\\
Interaction & & & $-0.78$ & $[-3.09,1.52]$ & $[-2.72,1.16]$ & inconclusive\\
\bottomrule
\end{tabular}
\end{table}

The error-structure statistics moved as in the earlier cohorts, and accuracy did not (\cref{tab:temp-perm-secondary}). At both temperatures, permutation co-occurred with lower wrong collision, more distinct wrong answers, slightly higher coverage, slightly lower single-call correctness, a lower estimated modal margin and a slightly lower certificate saving, counted in calls omitted in replay.

\begin{table}[!htbp]
\centering\footnotesize\setlength{\tabcolsep}{4pt}
\caption{Secondary quantities on the fresh MMLU-Pro hold-out (descriptive). Collision among errors uses the inputs with at least two wrong votes, which are Base and Permuted subsets of 276 and 296 inputs at $T=0.7$ and of 290 and 309 inputs at $T=1.0$. The certificate saving is the share of calls omitted in replay at $B=16$.}
\label{tab:temp-perm-secondary}
\begin{tabular}{lrrrr}
\toprule
 & \multicolumn{2}{c}{$T=0.7$} & \multicolumn{2}{c}{$T=1.0$}\\
\cmidrule(lr){2-3}\cmidrule(l){4-5}
Quantity & Base & Permuted & Base & Permuted\\
\midrule
All-call wrong collision & 0.122 & 0.109 & 0.113 & 0.106\\
Collision among errors & 0.693 & 0.634 & 0.676 & 0.599\\
Distinct wrong answers per input & 0.79 & 0.95 & 0.84 & 1.00\\
Coverage at 16 calls (\%) & 86.9 & 88.0 & 85.9 & 87.2\\
Single-call correctness (\%) & 68.7 & 67.8 & 67.6 & 66.9\\
Estimated modal margin & 0.740 & 0.711 & 0.716 & 0.696\\
Certificate saving (\%) & 35.1 & 34.0 & 34.1 & 33.4\\
\bottomrule
\end{tabular}
\end{table}

These results concern one model, one prompt format and two temperatures on a guarded, approximately category-balanced population. The three MMLU-Pro cohorts differ in population and implementation and do not provide a controlled estimate of cross-cohort attenuation.

\FloatBarrier

\section{Exact Decision Locking in Practice}
\label{app:stopping-experiments}
\label{app:stopping-cost}

\subsection{Stopping Replay}
\label{app:stopping-replay}

The locking certificate is replayed on 2,000 forced-choice MMLU-Pro trajectories and 5,000 GSM forecast trajectories, each with total budget 16. At every prefix the rule checks whether the current winner is the sole reachable endpoint winner. If no certificate appears before call 16, the full budget is used. The replay saves 28.35\% of MMLU-Pro calls and 30.26\% of GSM calls on average, while every stored fixed-budget endpoint is preserved (\cref{fig:lock-times,tab:decision-locking}).

\begin{figure}[!htbp]
\centering
\includegraphics[width=\linewidth]{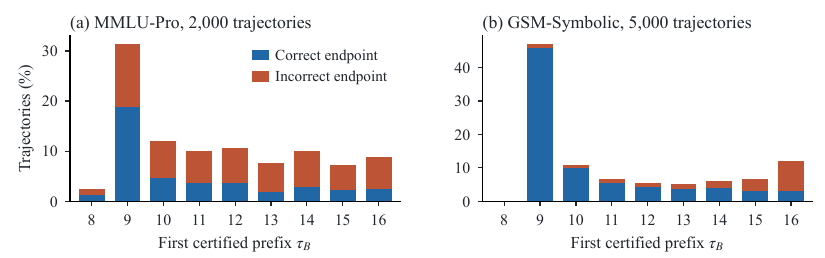}
\caption{Distribution of the first certified prefix under the exact locking rule, split by whether the fixed-budget endpoint is correct. Most trajectories are certified before the full 16-call budget. The bar at 16 counts trajectories not certified before the budget ends.}
\label{fig:lock-times}
\end{figure}

\begin{table}[!htbp]
\centering\small
\caption{Exact stopping replay with total budget 16. Output preservation follows from the certificate and is also verified on the stored trajectories.}
\label{tab:decision-locking}
\begin{tabular}{lrr}
\toprule
Metric & MMLU-Pro ($2{,}000$) & GSM ($5{,}000$)\\
\midrule
Locked by call 9 & $33.85\%$ & $47.10\%$\\
Locked by call 12 & $66.45\%$ & $70.06\%$\\
Not locked before call 16 & $8.85\%$ & $12.14\%$\\
First lock, quartiles & $9/11/14$ & $9/10/13$\\
Mean calls consumed & $11.46$ & $11.16$\\
Calls saved & $28.35\%$ & $30.26\%$\\
Answers changed from fixed-16 plurality & $0/2{,}000$ & $0/5{,}000$\\
\bottomrule
\end{tabular}
\end{table}

The certificate preserves incorrect endpoints as well as correct ones, because its guarantee concerns agreement with the fixed-budget decision. This distinction separates exact output preservation from accuracy-improving stopping objectives.

\subsection{Validation of the Lock-Time Law}
\label{app:lock-time-validation}

Each prediction estimates a call law from calls disjoint from the trajectory (raw empirical frequencies, abstention included) and simulates the certificate under that law (2,000 draws, fixed seeds). Simulation standard errors of the predicted savings are at most $0.01$ points. The predictions were stored before the replays, when the $B=16$ lock times of our own cohorts (\cref{tab:decision-locking}) were known and those of the public data were not. We evaluate the mean prediction error against a $\pm3$-point tolerance on its 95\% interval. The predicted lock-time distribution is evaluated against a $0.05$ tolerance on the 97.5\% bootstrap bound of its Kolmogorov distance from the observed one. For the asymptotic law, the tolerance is $0.03$ on the 95\% upper bound of the mean absolute deviation of per-item $\tau_B/B$ from $1/(1+\hat m)$ at $B=1024$. Intervals come from one cluster bootstrap per cohort (GSM template families, MMLU-Pro items, GSM8K problems).

\paragraph{Finite budgets.}
GSM-Symbolic supplies five independent 16-call sequences per checkpoint and item, numbered 0 to 4. The law sequence 0 predicts the forecast sequences 1--4 at $B=16$, and the concatenations $1\Vert2$ and $3\Vert4$ ($B=32$), $1\Vert2\Vert3$ ($B=48$) and $1\Vert2\Vert3\Vert4$ ($B=64$) are predicted from the remaining sequences. Two checkpoints generated sequences 3 and 4 in a different software environment (\cref{app:gsm-sensitivities}), so the rows above $B=16$ use the other three checkpoints. Forced-choice MMLU-Pro predicts each of its two sequences from the other. On GSM-Symbolic every mean error is at most $0.19$ points and every distributional bound at most $0.044$. On MMLU-Pro the mean errors are $-0.66$ $[-1.25,-0.03]$ and $+0.08$ points, and the distributional bounds are $0.057$ and $0.038$ (\cref{tab:lock-law-inhouse,fig:lock-law-cdfs}). Smoothing the estimated law with a Dirichlet$(1/2)$ prior shifts the predictions by $0.4$ to $2.0$ points on GSM-Symbolic and by 5 to 6 points on MMLU-Pro, where smoothing ten options estimated from 16 calls flattens the law and lowers the margin. The raw plug-in has the smallest error in all seven cells.

\begin{table}[!htbp]
\centering\footnotesize\setlength{\tabcolsep}{3pt}
\caption{Finite-budget validation (raw plug-in law from calls disjoint from the trajectory). Savings are percentages of the budget. Differences are observed minus predicted, in points, with cluster-bootstrap 95\% intervals. KS is the Kolmogorov distance between predicted and observed lock-time distributions with its 97.5\% bound. MAD is the mean absolute deviation of per-item $\tau_B/B$ from $1/(1+\hat m)$. The three checkpoints are Qwen3.5-4B, Gemma-4-E2B and Ministral-3-3B.}
\label{tab:lock-law-inhouse}
\resizebox{\linewidth}{!}{%
\begin{tabular}{@{}lrlrrrrlll@{}}
\toprule
Cohort & $B$ & Trajectory $\leftarrow$ law & Checkpoints & Traj. & Obs. & Pred. & Difference [95\% CI] & KS (bound) & MAD\\
\midrule
GSM-Symbolic & 16 & $1$--$4\leftarrow0$ & five$^{a}$ & 4,000 & 32.64 & 32.48 & $+0.16$ $[-0.17,+0.47]$ & 0.012 (0.022) & 0.067\\
GSM-Symbolic & 32 & $1\Vert2\leftarrow0,3,4$ & three & 750 & 39.06 & 39.01 & $+0.04$ $[-0.30,+0.39]$ & 0.010 (0.028) & 0.040\\
GSM-Symbolic & 32 & $3\Vert4\leftarrow0,1,2$ & three & 750 & 39.06 & 39.04 & $+0.02$ $[-0.32,+0.33]$ & 0.011 (0.029) & 0.039\\
GSM-Symbolic & 48 & $1\Vert2\Vert3\leftarrow0,4$ & three & 750 & 40.02 & 39.83 & $+0.19$ $[-0.11,+0.49]$ & 0.013 (0.031) & 0.033\\
GSM-Symbolic & 64 & $1\Vert2\Vert3\Vert4\leftarrow0$ & three & 750 & 40.32 & 40.20 & $+0.12$ $[-0.28,+0.51]$ & 0.018 (0.044) & 0.034\\
MMLU-Pro & 16 & seq.\ 0 $\leftarrow$ seq.\ 1 & five & 1,000 & 28.15 & 28.81 & $-0.66$ $[-1.25,-0.03]$ & 0.039 (0.057) & 0.092\\
MMLU-Pro & 16 & seq.\ 1 $\leftarrow$ seq.\ 0 & five & 1,000 & 28.54 & 28.47 & $+0.08$ $[-0.50,+0.64]$ & 0.015 (0.038) & 0.088\\
\bottomrule
\end{tabular}%
}
\par\smallskip\raggedright\footnotesize $^{a}$Sequences 1 and 2 of every checkpoint and sequences 3 and 4 of the other three checkpoints. Sequences 3 and 4 of 0.8B and 2B, generated in the second environment (1,000 trajectories), give $-0.18$ $[-0.80,+0.48]$.
\end{table}

\begin{figure}[!htbp]
\centering
\includegraphics[width=\linewidth]{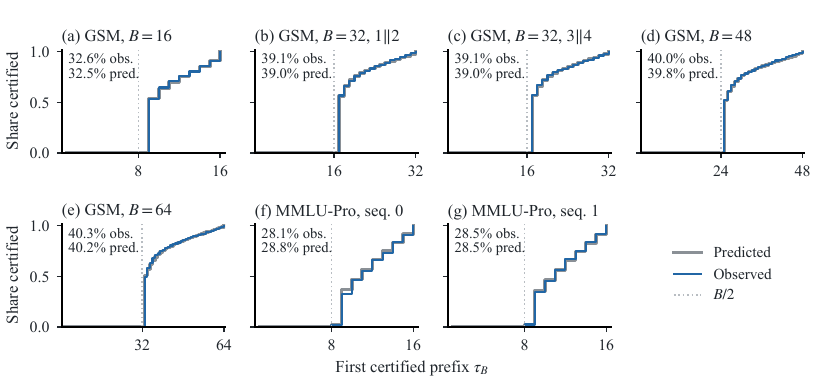}
\caption{Cumulative distribution of the first certified prefix, predicted from an independently estimated law (gray) and observed (blue), for the seven cells of \cref{tab:lock-law-inhouse}, in table order. Dotted lines mark $B/2$, below which no prefix can be certified.}
\label{fig:lock-law-cdfs}
\end{figure}

\paragraph{Large budgets on public data.}
The Large Language Monkeys release \citep{brown2024largelanguagemonkeys} provides 10,000 samples per problem for 127 GSM8K problems from Llama-3-8B-Instruct and Llama-3-70B-Instruct (temperature $0.6$, few-shot completion prompting). Answer identities use the dataset's extractor followed by exact-rational canonicalization, and our correctness labels match the dataset's on all but one of 2,540,000 samples. Samples 1--2,000 of each problem estimate the law, and samples 2,001--10,000 form disjoint trajectories at $B\in\{16,32,64,\ldots,1024\}$ under the open-universe rule. Of the sample texts, 13.3\% and 16.4\% duplicate another sample of the same problem (short identical solutions), and the two sample blocks differ about as much as a permutation reference, consistent with iid sampling. Every mean error is at most $0.09$ points and every distributional bound at most $0.043$ (\cref{tab:lock-law-monkeys}). The savings rise from 34.9\% to 39.6\% (8B) and from 41.4\% to 47.1\% (70B), against population averages of $\hat m/(1+\hat m)$ of 39.5\% and 47.1\%, and the per-problem deviation from $1/(1+\hat m)$ at $B=1024$ is $0.0041$ and $0.0020$ (\cref{fig:lock-law-scatter}). On problems whose estimated mode and runner-up are confidently separated, the mean observed variance of $\tau_B$ at $B=1024$ is 150 and 130 against 151 and 155 from \cref{thm:lock-time-law}(ii), a descriptive agreement. Certified outputs equal the fixed-budget outputs in every trajectory of the in-house and public replays. These are call savings under iid sampling from one fixed prompt. Latency is measured separately in \cref{app:physical-execution}.

\begin{table}[!htbp]
\centering\footnotesize\setlength{\tabcolsep}{3pt}
\caption{Large-budget validation on public GSM8K (127 problems). Trajectories are disjoint blocks of samples 2,001--10,000, and the law comes from samples 1--2,000. Savings are percentages of the budget, and the asymptote is the population average of $\hat m/(1+\hat m)$. MAD is the mean absolute deviation of per-problem $\tau_B/B$ from $1/(1+\hat m)$. KS is the Kolmogorov distance between predicted and observed lock-time distributions with its 97.5\% bootstrap bound. The last column compares the mean observed within-problem variance of $\tau_B$ with $B\hat\sigma^2/(1+\hat m)^3$ on problems with a confidently separated runner-up (69 for 8B, 27 for 70B).}
\label{tab:lock-law-monkeys}
\resizebox{\linewidth}{!}{%
\begin{tabular}{@{}lrrrrlrlll@{}}
\toprule
Model & $B$ & Traj. & Obs. & Pred. & Difference [95\% CI] & Asymptote & MAD [95\% CI] & KS (bound) & Var.\ obs./CLT\\
\midrule
Llama-3-8B & 16 & 63,500 & 34.90 & 34.85 & $+0.05$ $[-0.03,+0.14]$ & 39.53 & 0.0565 $[0.0533,0.0597]$ & 0.004 (0.008) & 1.5 / 2.4\\
 & 32 & 31,750 & 37.34 & 37.29 & $+0.05$ $[-0.03,+0.13]$ & & 0.0297 $[0.0275,0.0319]$ & 0.007 (0.012) & 3.4 / 4.7\\
 & 64 & 15,875 & 38.55 & 38.48 & $+0.07$ $[-0.02,+0.16]$ & & 0.0155 $[0.0138,0.0174]$ & 0.006 (0.013) & 7.5 / 9.5\\
 & 128 & 7,874 & 39.10 & 39.03 & $+0.07$ $[-0.02,+0.17]$ & & 0.0087 $[0.0074,0.0102]$ & 0.006 (0.014) & 16.5 / 18.9\\
 & 256 & 3,937 & 39.38 & 39.29 & $+0.09$ $[-0.01,+0.19]$ & & 0.0056 $[0.0045,0.0069]$ & 0.012 (0.024) & 36.1 / 37.8\\
 & 512 & 1,905 & 39.48 & 39.42 & $+0.05$ $[-0.04,+0.15]$ & & 0.0042 $[0.0033,0.0053]$ & 0.010 (0.024) & 68.0 / 75.7\\
 & 1024 & 889 & 39.55 & 39.48 & $+0.07$ $[-0.04,+0.18]$ & & 0.0041 $[0.0032,0.0051]$ & 0.020 (0.043) & 149.9 / 151.4\\
\midrule
Llama-3-70B & 16 & 63,500 & 41.36 & 41.37 & $-0.02$ $[-0.07,+0.03]$ & 47.14 & 0.0581 $[0.0555,0.0605]$ & 0.003 (0.006) & 1.9 / 2.4\\
 & 32 & 31,750 & 44.28 & 44.27 & $+0.01$ $[-0.04,+0.06]$ & & 0.0287 $[0.0273,0.0300]$ & 0.002 (0.005) & 4.4 / 4.8\\
 & 64 & 15,875 & 45.73 & 45.73 & $0.00$ $[-0.05,+0.05]$ & & 0.0143 $[0.0136,0.0150]$ & 0.002 (0.007) & 9.2 / 9.7\\
 & 128 & 7,874 & 46.46 & 46.45 & $+0.01$ $[-0.04,+0.06]$ & & 0.0073 $[0.0069,0.0077]$ & 0.003 (0.011) & 19.6 / 19.3\\
 & 256 & 3,937 & 46.83 & 46.81 & $+0.02$ $[-0.04,+0.08]$ & & 0.0040 $[0.0036,0.0044]$ & 0.006 (0.015) & 36.4 / 38.7\\
 & 512 & 1,905 & 46.99 & 46.98 & $+0.01$ $[-0.04,+0.07]$ & & 0.0025 $[0.0022,0.0030]$ & 0.008 (0.022) & 77.5 / 77.4\\
 & 1024 & 889 & 47.08 & 47.06 & $+0.02$ $[-0.04,+0.09]$ & & 0.0020 $[0.0015,0.0026]$ & 0.010 (0.032) & 129.9 / 154.7\\
\bottomrule
\end{tabular}%
}
\end{table}

\begin{figure}[!htbp]
\centering
\includegraphics[width=\linewidth]{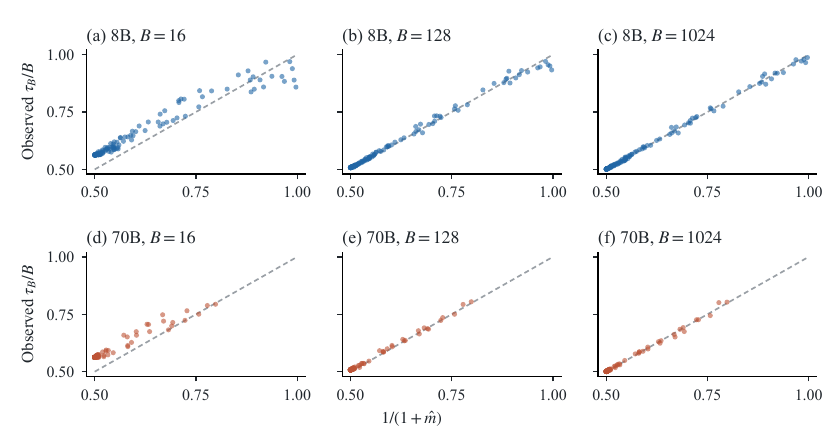}
\caption{Per problem on public GSM8K: observed $\tau_B/B$ (mean over the problem's trajectories) against $1/(1+\hat m)$ at $B=16$, $128$ and $1024$, for Llama-3-8B-Instruct (top) and Llama-3-70B-Instruct (bottom). \Cref{tab:lock-law-monkeys} reports all seven budgets. Deviations from the asymptote shrink as the budget grows.}
\label{fig:lock-law-scatter}
\end{figure}

\paragraph{A larger checkpoint and a harder public task.}
For Qwen3.5-27B under the forced-choice protocol (\cref{app:scope-scale}), predicting each sequence from the other with the MMLU-Pro procedure above gives observed savings of 38.78\% and 38.41\% against predictions of 38.90\% and 39.06\% ($-0.12$ $[-1.10,+0.80]$ and $-0.65$ $[-1.53,+0.20]$ points). The two distributional bounds are $0.048$ and $0.063$. Dirichlet$(1/2)$ smoothing under-predicts the savings by $7.39$ $[6.41,8.32]$ and $6.85$ $[5.91,7.75]$ points, because smoothing ten options estimated from 16 calls flattens the law. The 27B replay locks 75.0\% of trajectories by call 9, saves 38.6\% of calls and changes no output. On public MATH, the procedure of \cref{tab:lock-law-monkeys} predicts every raw-law saving within $0.13$ points, and every distributional bound is at most $0.043$ (\cref{tab:lock-law-math}). The savings are smaller than on GSM8K. No certified output differs from its fixed-budget output in any of the 253,440 trajectories, under either identity rule. The Dirichlet$(1/2)$, dispersed missing-mass and cleaned-string sensitivities stay within both tolerances, with point discrepancies of at most 0.49 points. A single concentrated unseen rival over-predicts the savings by 0.8 to 1.7 points, and its distributional bound exceeds $0.05$ in every cell.

\begin{table}[!htbp]
\centering\footnotesize\setlength{\tabcolsep}{3pt}
\caption{Large-budget validation on public MATH for 128 problems, with laws from samples 1--2,000, disjoint trajectories from samples 2,001--10,000 and the conventions of \cref{tab:lock-law-monkeys}. Mean $\hat m$ is 0.24 for 8B and 0.45 for 70B. The variance column uses problems with a confidently separated runner-up (98 and 86).}
\label{tab:lock-law-math}
\resizebox{\linewidth}{!}{%
\begin{tabular}{@{}lrrrrlrlll@{}}
\toprule
Model & $B$ & Traj. & Obs. & Pred. & Difference [95\% CI] & Asymptote & MAD [95\% CI] & KS (bound) & Var.\ obs./CLT\\
\midrule
Llama-3-8B & 16 & 64{,}000 & 13.28 & 13.19 & $+0.09$ $[-0.02,+0.21]$ & 15.75 & 0.0419 $[0.0375,0.0464]$ & 0.006 (0.010) & 1.4 / 3.3\\
 & 64 & 16{,}000 & 15.17 & 15.05 & $+0.13$ $[-0.03,+0.28]$ & & 0.0178 $[0.0159,0.0198]$ & 0.009 (0.018) & 9.3 / 13.1\\
 & 256 & 3{,}968 & 15.70 & 15.57 & $+0.12$ $[-0.07,+0.30]$ & & 0.0100 $[0.0086,0.0116]$ & 0.012 (0.029) & 46.5 / 52.3\\
 & 1024 & 896 & 15.82 & 15.72 & $+0.10$ $[-0.09,+0.30]$ & & 0.0088 $[0.0074,0.0101]$ & 0.017 (0.043) & 203.1 / 209.3\\
\midrule
Llama-3-70B & 16 & 64{,}000 & 22.97 & 22.85 & $+0.12$ $[+0.04,+0.21]$ & 26.25 & 0.0462 $[0.0417,0.0506]$ & 0.004 (0.008) & 1.5 / 3.1\\
 & 64 & 16{,}000 & 25.60 & 25.51 & $+0.09$ $[-0.02,+0.20]$ & & 0.0154 $[0.0135,0.0176]$ & 0.005 (0.013) & 8.7 / 12.3\\
 & 256 & 3{,}968 & 26.20 & 26.08 & $+0.12$ $[-0.01,+0.25]$ & & 0.0072 $[0.0060,0.0085]$ & 0.011 (0.023) & 42.1 / 49.2\\
 & 1024 & 896 & 26.33 & 26.22 & $+0.12$ $[-0.03,+0.25]$ & & 0.0057 $[0.0046,0.0069]$ & 0.017 (0.041) & 176.0 / 196.7\\
\bottomrule
\end{tabular}%
}
\end{table}

\subsection{Token Accounting in Thinking Mode}
\label{app:thinking-tokens}

The replays above count calls. With reasoning models most of a call's cost is its thinking tokens, whose number varies across calls, so a saving in calls need not equal a saving in tokens. We therefore generated fresh thinking-mode trajectories and replayed the certificate on them. Qwen3-14B (primary) and Qwen3-8B \citep{yang2025qwen3} answered 100 problems of MATH-500 \citep{hendrycks2021math,lightman2024lets}, 20 per difficulty level, in thinking mode. Generation used the native chat template with thinking enabled, temperature $0.6$, top-$p$ $0.95$, top-$k$ $20$ and at most 16,384 generated tokens. Each problem and model has two separately generated 32-call sequences, produced offline with vLLM 0.13.0 \citep{kwon2023vllm} and a fixed seed per call. A call is valid when its thinking block closes and a final boxed answer follows it. A call that reaches the length limit or has no final answer abstains and is still charged. Answer identities are the boxed strings after a type-aware normalization fixed on development problems before any test call. Merging numerically equal answers is a sensitivity analysis. The certificate is the open-universe rule of \cref{sec:stopping-replay}. The tokens it omits at budget $B$ are all generated tokens of the calls after the first certified prefix $\tau_B$, thinking and answer together.

The analysis was pre-specified before generation. Its three quantities are the pooled share of generated tokens omitted at $B=32$, the calibration of the lock-time law and the difference between the token share and the call share. The token share is total omitted over total generated tokens across all test sequences. For calibration, the law estimated from one sequence predicts the other sequence's lock times, evaluated with the tolerances of \cref{app:lock-time-validation} except that the distributional bound is the one-sided 95\% bootstrap percentile. If the (answer, length) pairs of a problem's calls are independent and identically distributed, the expected omitted tokens of a sequence equal $\mathbb{E}[L]\,\mathbb{E}[B-\tau_B]$, with $L$ the length of a call. This identity holds for expected totals and does not extend to the ratio of the two shares. Intervals resample whole problems, both sequences together, within difficulty levels (2,000 draws). On 12 separate development problems, batching had no material effect and concurrently generated calls showed no material dependence (largest excess total variation $0.0014$). Simulations of the full analysis at 100 problems gave interval coverage of 92--96\%, but the distributional tolerance had limited power. It was never met under a parametric law family, even with the true law, and was met in 57--86\% of datasets under the concentrated development laws. A miss on it is therefore read as inconclusive. The lock-time law and the length identity treat a problem's calls as independent (answer, length) draws. The replayed token counts do not rely on this assumption.

\paragraph{Results.}
Each model produced 6,400 calls. Of these, 2.0\% (14B) and 2.7\% (8B) reached the length limit, over 90\% of them on ten problems. The certified output equals the fixed-budget plurality in every sequence. At $B=32$ the certificate omits 41.9\% $[37.2, 45.6]$ of the generated tokens of Qwen3-14B and 40.6\% $[35.8, 44.6]$ of Qwen3-8B, about 54,000 and 59,000 tokens per 32-call sequence (\cref{tab:thinking-tokens}). The token share lies 3.2 $[0.7, 6.5]$ and 3.5 $[0.9, 6.6]$ points below the call share. The estimated gap is predominantly attributable to length weighting between sequences ($-3.3$ points $[-6.6, -0.7]$ for 14B). Problems with longer calls, often level-5 problems with truncations, lock later. The within-sequence component, which compares the omitted calls with the rest of their sequence, is $+0.1$ points $[-0.2, +0.4]$ and compatible with zero. In an exploratory breakdown by level, the certificate omits 47\% of the tokens at level 1 and 36\% at level 5. The mean prediction error stays within the $\pm3$-point tolerance for both models at every budget, law variant and identity scheme. In the primary cell (Qwen3-14B, $B=32$) the observed saving exceeds the prediction by $0.16$ points $[-0.03, +0.39]$, and the Kolmogorov distance is $0.009$ with bound $0.021$. Every distributional bound is below $0.05$ in \cref{tab:thinking-tokens} and in every sensitivity cell except the Dirichlet-smoothed law for 14B at $B=32$ (bound $0.051$ in both identity schemes). The answer laws are close to degenerate, however. Of the sequences, 91.5\% (14B) and 86.5\% (8B) lock at call 17, the earliest possible at $B=32$. This cohort therefore tests the law and the token accounting on concentrated answer laws, and problems with contested answers remain untested.

On such laws statistical stopping saves considerably more (\cref{tab:thinking-context}). ASC at $0.95$, recomputed post hoc with the posterior tail of \citet{aggarwal2023adaptive}, stops after four calls on most sequences. It omits 84.5\% of the calls and 78.4\% of the tokens of Qwen3-14B, and ESC with windows of eight omits 70.0\% and 62.0\%. Neither changed the endpoint of any of the 200 sequences per model (100 problems), but only the certificate guarantees this. Its accuracy therefore equals the fixed-budget accuracy (88.5\% and 87.5\% at $B=32$). All savings count generated tokens in replay and do not measure latency.

\begin{table}[!htbp]
\centering\footnotesize\setlength{\tabcolsep}{3pt}
\caption{Thinking-mode token accounting on 100 MATH-500 problems, with two 32-call sequences per problem and model (6,400 calls per model). Shares are percentages of the generated tokens or calls that the certificate omits at budget $B$. Token minus call and observed minus predicted saving are in points. Intervals are 95\% problem-bootstrap intervals within difficulty levels. KS is the Kolmogorov distance between the predicted and observed lock-time distributions with its one-sided 95\% bootstrap bound, the pre-specified level for this study. Predictions use the raw law of the other sequence, and answers are normalized boxed strings.}
\label{tab:thinking-tokens}
\resizebox{\linewidth}{!}{%
\begin{tabular}{@{}lrlllll@{}}
\toprule
Model & $B$ & Tokens omitted & Calls omitted & Token $-$ call & Obs.\ $-$ pred. & KS (bound)\\
\midrule
Qwen3-14B & 32 & 41.9 $[37.2, 45.6]$ & 45.1 $[43.6, 46.3]$ & $-3.2$ $[-6.5, -0.7]$ & $+0.16$ $[-0.03, +0.39]$ & 0.009 (0.021)\\
 & 24 & 41.3 $[37.0, 44.8]$ & 44.1 $[42.7, 45.3]$ & $-2.8$ $[-6.0, -0.4]$ & $+0.16$ $[-0.08, +0.43]$ & 0.011 (0.024)\\
 & 16 & 38.9 $[34.4, 42.6]$ & 41.9 $[40.3, 43.1]$ & $-3.0$ $[-6.2, -0.4]$ & $-0.08$ $[-0.50, +0.22]$ & 0.008 (0.026)\\
\midrule
Qwen3-8B & 32 & 40.6 $[35.8, 44.6]$ & 44.1 $[42.3, 45.6]$ & $-3.5$ $[-6.6, -0.9]$ & $-0.09$ $[-0.23, +0.03]$ & 0.011 (0.024)\\
 & 24 & 40.0 $[35.4, 43.8]$ & 43.3 $[41.4, 44.8]$ & $-3.3$ $[-6.3, -0.8]$ & $+0.05$ $[-0.21, +0.30]$ & 0.011 (0.029)\\
 & 16 & 38.4 $[34.0, 42.1]$ & 41.3 $[39.6, 42.7]$ & $-2.9$ $[-5.8, -0.5]$ & $+0.05$ $[-0.27, +0.37]$ & 0.009 (0.023)\\
\bottomrule
\end{tabular}%
}
\end{table}

\begin{table}[!htbp]
\centering\footnotesize\setlength{\tabcolsep}{4pt}
\caption{Exact locking and statistical stopping on the same thinking-mode sequences at $B=32$ (200 sequences per model). Shares are percentages of calls and generated tokens omitted. ASC uses the posterior tail of \citet{aggarwal2023adaptive}, computed post hoc. ESC uses nonoverlapping windows of eight calls, and a window with an abstention does not trigger. A change is an endpoint that differs from fixed-budget plurality. Only the certificate excludes changes by construction.}
\label{tab:thinking-context}
\begin{tabular}{@{}llrrrl@{}}
\toprule
Model & Rule & Calls omitted & Tokens omitted & Changed endpoints & Guarantee\\
\midrule
Qwen3-14B & Certificate & 45.1 & 41.9 & 0 & yes\\
 & ASC 0.95 & 84.5 & 78.4 & 0 & no\\
 & ESC, window 8 & 70.0 & 62.0 & 0 & no\\
\midrule
Qwen3-8B & Certificate & 44.1 & 40.6 & 0 & yes\\
 & ASC 0.95 & 83.0 & 77.1 & 0 & no\\
 & ESC, window 8 & 67.0 & 58.0 & 0 & no\\
\bottomrule
\end{tabular}
\end{table}

\subsection{Physical Execution and Baselines}
\label{app:physical-execution}
\label{app:stopping-quality}
\label{app:stopping-accounting}

The physical experiment compares fixed-16, exact locking, ASC, ESC, and Bayesian stopping on 32 MMLU-Pro items per checkpoint for Qwen3.5-2B and 4B, using microbatch sizes one and four (\cref{fig:physical-stopping,tab:physical-stopping-ci}). Request time includes generation and the stopping decision. The same-shape fixed-16 run is the latency baseline, while an independent batch-16 run provides a fully parallel reference.

\begin{figure}[!htbp]
\centering
\includegraphics[width=0.88\linewidth]{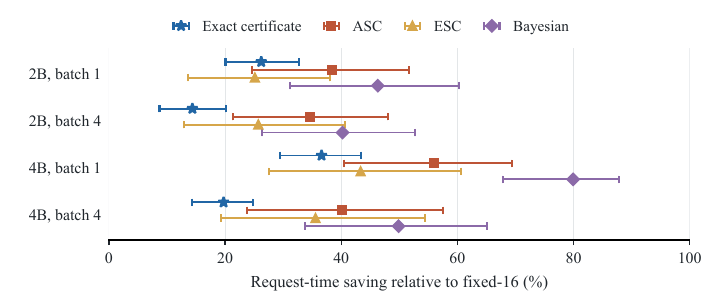}
\caption{Physical request-time savings on 32 paired MMLU-Pro items per model and batch size, relative to the same-batch fixed-16 baseline. Bars are paired item-bootstrap 95\% intervals, and each item has one timing. The larger output audit in \cref{tab:stopping-quality} uses a separate cohort.}
\label{fig:physical-stopping}
\end{figure}

\paragraph{Stopping baselines.}
ASC compares the two largest valid vote counts $a,b$ and stops when $\Pr[\mathrm{Beta}(a+1,b+1)>1/2]\geq0.95$, using the exact binomial integration corresponding to the criterion of \citet{aggarwal2023adaptive}. ESC uses consecutive nonoverlapping four-call windows and stops when all four answers in a complete window are the same valid answer \citep{li2024esc}. Windows containing abstentions do not trigger, and at the cap the method otherwise returns plurality. Bayesian stopping follows the categorical-prior procedure of \citet{huang2026bayesian} with threshold $0.95$ and an empirical prior fitted without gold on 32 separate development items with eight responses each. The prior parameter and lookup table are fixed before test generation. Stopping is evaluated only after a completed microbatch, and every issued call, invalid completion, and batch overshoot is charged. Bayesian stopping is used only for MMLU-Pro, whose answer alphabet is known and finite.

On the larger 1,200-item native MMLU-Pro replay, exact locking changes no endpoint by construction. For 4B, Bayesian stopping changes 22 endpoints, ASC changes 1, and ESC changes 6, with corresponding accuracy differences of $-0.75$, $-0.083$, and $-0.083$ points. For 2B, the mismatch counts are 11, 1, and 1. These comparisons illustrate the distinction between an exact certificate and statistical stopping rules that can save more calls while occasionally changing the fixed-budget output.

\begin{table}[!htbp]
\centering\small
\caption{Separate output audit: 1,200 native 4B MMLU-Pro items, sequential budget 16. Changes are relative to fixed-16 plurality, and accuracy differences are percentage points with pointwise 95\% intervals.}
\label{tab:stopping-quality}
\begin{tabular}{@{}lrrrr@{}}
\toprule
Method & Certificate & ASC & ESC & Bayesian\\
\midrule
Changed outputs & 0 & 1 & 6 & 22\\
Accuracy difference & $0$ (guaranteed) & $-0.083$ & $-0.083$ & $-0.750$\\
95\% interval & $[0,0]$ & $[-0.25,0]$ & $[-0.417,0.25]$ & $[-1.42,-0.17]$\\
\bottomrule
\end{tabular}
\end{table}

\begin{table}[!htbp]
\centering\footnotesize\setlength{\tabcolsep}{3pt}
\caption{Paired request-time savings relative to the same-shape fixed-16 baseline. Each model/shape has 32 items, and intervals are 95\% item-bootstrap intervals.}
\label{tab:physical-stopping-ci}
\begin{tabular}{llrrrr}
\toprule
Model & Batch & Certificate & ASC .95 & ESC4 & Bayesian\\
\midrule
2B & 1 & $26.23\ [20.07,32.73]$ & $38.35\ [24.61,51.66]$ & $25.14\ [13.67,38.03]$ & $46.29\ [31.21,60.31]$\\
2B & 4 & $14.38\ [8.71,20.11]$ & $34.69\ [21.35,47.99]$ & $25.72\ [12.86,40.64]$ & $40.23\ [26.40,52.77]$\\
4B & 1 & $36.61\ [29.49,43.45]$ & $55.94\ [40.51,69.37]$ & $43.33\ [27.55,60.63]$ & $79.93\ [67.79,87.77]$\\
4B & 4 & $19.74\ [14.38,24.80]$ & $40.20\ [23.85,57.53]$ & $35.57\ [19.35,54.50]$ & $49.88\ [33.71,65.11]$\\
\bottomrule
\end{tabular}
\end{table}

For the exact certificate, corresponding mean output-token savings are 25.94\% and 12.80\% for 2B at batch sizes one and four, and 35.88\% and 20.93\% for 4B. Fully parallel fixed-16 is faster in wall-clock time than the sequential and microbatched stopping methods in this experiment, emphasizing that omitted calls and end-to-end latency are distinct cost measures. On the 32 items for 2B, first-use Bayesian stopping consumes more input and output tokens than fixed-16 because of its prior construction, which later items can amortize. One timing per item on one backend limits transfer.

\paragraph{Fresh three-task execution.}
The fresh cohort of \cref{app:fresh-mechanism} adds 48 online inputs per task, executed with each task's development-selected arm and the necessary-condition selector (High temperature for Word16, Permuted for MMLU-Pro, Base for shortest paths). Each input receives fixed-16, the certificate, ASC at $0.95$ and $0.99$ and a four-call ESC window at microbatch four, plus fully parallel fixed-16. The first sixteen inputs per task also receive fixed and certificate runs at microbatch one. Input and method order are randomized, methods with the same input and batch size share batch seeds, and every issued call is charged. The pre-specified cost contrasts are the certificate-to-fixed ratios of mean batch-four request time, with 98.33\% paired input-bootstrap intervals (10,000 draws): $0.904$ $[0.858,0.947]$ on Word16, $0.784$ $[0.691,0.858]$ on MMLU-Pro and $0.845$ $[0.803,0.889]$ on shortest paths, all below one. All 192 certificate executions (144 at batch four, 48 at batch one) return their fixed-cap outputs, and the batch-one ratios are $0.76$, $0.71$ and $0.81$. ASC at $0.95$ is faster than the certificate on all three tasks, and none of the 432 ASC and ESC executions changed an observed output, although neither rule guarantees this (\cref{tab:fresh-online}). Fully parallel fixed-16 is 2.2 to 3.2 times faster than the certificate at batch four, on different batch trajectories.

\begin{table}[!htbp]
\centering\footnotesize\setlength{\tabcolsep}{4pt}
\caption{Fresh physical execution, 48 inputs per task: accuracy (\%), mean calls and mean request time (s) per input. Methods other than parallel fixed-16 run at microbatch four.}
\label{tab:fresh-online}
\begin{tabular}{@{}lrrrrrrrrr@{}}
\toprule
 & \multicolumn{3}{c}{Word16} & \multicolumn{3}{c}{MMLU-Pro} & \multicolumn{3}{c}{Shortest path}\\
\cmidrule(lr){2-4}\cmidrule(lr){5-7}\cmidrule(l){8-10}
Method & Acc. & Calls & Time & Acc. & Calls & Time & Acc. & Calls & Time\\
\midrule
Fixed-16 & 95.83 & 16.00 & 100.43 & 85.42 & 16.00 & 113.09 & 93.75 & 16.00 & 342.62\\
Certificate & 95.83 & 14.50 & 90.82 & 85.42 & 11.75 & 88.67 & 93.75 & 13.50 & 289.50\\
ASC 0.95 & 95.83 & 13.50 & 86.19 & 85.42 & 6.33 & 58.00 & 93.75 & 10.92 & 235.03\\
ASC 0.99 & 95.83 & 15.17 & 94.94 & 85.42 & 10.33 & 82.78 & 93.75 & 13.00 & 279.66\\
ESC window 4 & 95.83 & 15.42 & 97.23 & 85.42 & 6.92 & 65.20 & 93.75 & 13.58 & 292.62\\
Parallel fixed-16 & 97.92 & 16.00 & 35.71 & 85.42 & 16.00 & 39.74 & 91.67 & 16.00 & 89.39\\
\bottomrule
\end{tabular}
\end{table}

\subsection{Plug-in Stochastic Curtailment}
\label{app:curtailment}

The exact certificate coincides with the $\epsilon=0$ member of a family of rules that stop at the first prefix at which the estimated probability $\rho_N$ that the fixed-budget winner differs from the current winner is at most $\epsilon$, provided the estimated law gives every answer positive probability, as the Dirichlet laws below do. Without full support, an estimate can assign zero probability to a continuation that changes the output. Under the true iid law such a rule changes the fixed-budget output with probability at most $\epsilon$. With a plug-in law $\hat p_T$ the bound becomes $\min\{1,\epsilon+\mathbb{E}[(B-T)\,\mathrm{TV}(p,\hat p_T)]\}$, which is loose. On the 2,000 forced-choice MMLU-Pro trajectories (known option universe, no abstentions), the law is the Dirichlet$(\alpha)$ posterior mean from the prefix over the item's options ($\alpha=1/2$, with $\alpha\in\{0.1,1\}$ as sensitivities), and $\rho_N$ is computed exactly by the coefficient extraction of \cref{app:endpoint-laws} with the current winner in the role of gold. Re-implemented ASC, ESC, a generic Dirichlet posterior-mode rule and a majority lock are evaluated after every call and return fixed-16 plurality at the cap.

Only the certificate and the majority lock guarantee an unchanged output, although ASC at $0.99$ and ESC with eight-call windows also changed no output in this replay (\cref{fig:curtailment,tab:curtailment}). Among these zero-change rules the certificate saves the most: 28.4\% of calls, against 25.4\% for ASC at $0.99$, 22.3\% for the majority lock, which is exact and so by \cref{rem:earliest-exact-stopping} can never stop earlier, and 15.3\% for eight-call ESC. The ASC and ESC orderings are empirical: in the fresh online study, ASC at $0.99$ used fewer calls than the certificate on two of three tasks, also without an observed change (\cref{tab:fresh-online}). The plug-in family extends the certificate along the same trade-off, from 55.8\% saved with 2.1\% changed outputs at $\epsilon=0.005$ to 83.4\% saved with 17.6\% changed at $\epsilon=0.2$. It dominates the posterior-mode rule and the majority lock in this comparison but not ASC or ESC, whose $0.95$ and four-call settings save 40.8\% and 37.2\% of calls while changing 0.30\% and 0.45\% of outputs. The nominal $\epsilon$ is not a calibrated risk for this estimator: realized change rates exceed $\epsilon$ at every tested threshold up to $0.05$. A stronger prior ($\alpha=1$) halves the change rate at $\epsilon=0.005$, and a weak one ($\alpha=0.1$) stops after the first call on every trajectory, because a single vote already puts the estimated risk below every tested $\epsilon$. These are retrospective trade-offs on one benchmark, and within the plug-in family only the $\epsilon=0$ member, the exact certificate, needs no estimated law.

\begin{figure}[!htbp]
\centering
\includegraphics[width=\linewidth]{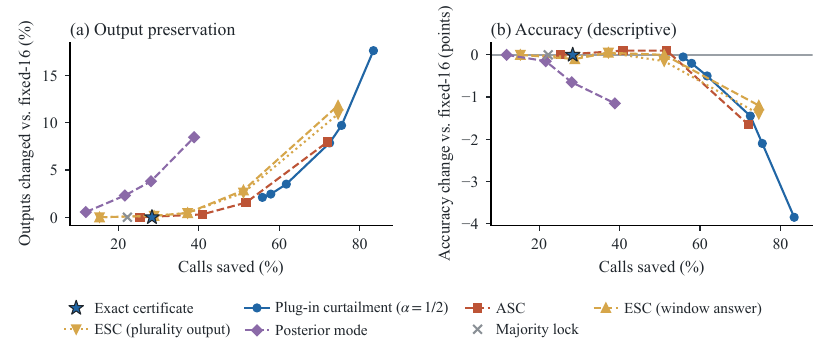}
\caption{Calls saved against changed outputs (a) and against accuracy change (b), relative to fixed-16 plurality, on 2,000 forced-choice MMLU-Pro trajectories: the plug-in curtailment family over $\epsilon$ ($\alpha=1/2$), re-implemented ASC and ESC, a Dirichlet$(1/2)$ posterior-mode rule and a majority lock. The two ESC curves stop at the same call and return either the answer of the unanimous window or the plurality of all calls so far. The exact certificate is the $\epsilon=0$ member of the family. Point estimates, with intervals in \cref{tab:curtailment} for the settings it lists.}
\label{fig:curtailment}
\end{figure}

\begin{table}[!htbp]
\centering\footnotesize\setlength{\tabcolsep}{4pt}
\caption{Stopping rules on 2,000 forced-choice MMLU-Pro trajectories, $B=16$: calls saved, changed outputs and accuracy change relative to fixed-16 plurality, with joint item-bootstrap 95\% intervals (at zero observed changes, the upper bound is the conservative cluster-level bound $1.83\%$).}
\label{tab:curtailment}
\resizebox{\linewidth}{!}{%
\begin{tabular}{@{}llrlll@{}}
\toprule
Rule & Parameter & Mean stop & Saved \% & Changed \% & Accuracy change (points)\\
\midrule
Exact certificate & $\epsilon=0$ & 11.46 & 28.35 $[27.28,29.33]$ & 0.00 $[0.00,1.83]$ & 0.00\\
Majority lock & $>B/2$ & 12.44 & 22.25 $[20.86,23.57]$ & 0.00 $[0.00,1.83]$ & 0.00\\
Plug-in curtailment, $\alpha=1/2$ & $\epsilon=0.005$ & 7.07 & 55.80 $[53.84,57.69]$ & 2.10 $[1.45,2.80]$ & $-0.05$ $[-0.55,+0.45]$\\
 & $\epsilon=0.02$ & 6.12 & 61.76 $[59.99,63.47]$ & 3.50 $[2.70,4.35]$ & $-0.50$ $[-1.10,+0.10]$\\
 & $\epsilon=0.05$ & 4.40 & 72.51 $[71.03,73.96]$ & 7.85 $[6.55,9.15]$ & $-1.45$ $[-2.40,-0.60]$\\
 & $\epsilon=0.2$ & 2.65 & 83.42 $[82.55,84.28]$ & 17.60 $[15.80,19.60]$ & $-3.85$ $[-5.35,-2.45]$\\
Plug-in curtailment, $\alpha=1$ & $\epsilon=0.005$ & 7.97 & 50.19 $[48.39,51.87]$ & 1.05 $[0.65,1.50]$ & $0.00$ $[-0.30,+0.30]$\\
Plug-in curtailment, $\alpha=0.1$ & every tested $\epsilon>0$ & 1.00 & 93.75 & 30.70 $[28.30,33.10]$ & $-5.60$ $[-7.60,-3.75]$\\
ASC & $\theta=0.90$ & 7.72 & 51.73 $[49.50,53.78]$ & 1.55 $[1.05,2.15]$ & $+0.10$ $[-0.30,+0.55]$\\
 & $\theta=0.95$ & 9.47 & 40.83 $[38.49,43.03]$ & 0.30 $[0.10,0.55]$ & $+0.10$ $[0.00,+0.25]$\\
 & $\theta=0.99$ & 11.94 & 25.35 $[23.18,27.48]$ & 0.00 $[0.00,1.83]$ & 0.00\\
ESC & $w=4$ & 10.05 & 37.21 $[34.71,39.55]$ & 0.45 $[0.20,0.75]$ & $+0.05$ $[-0.10,+0.25]$\\
 & $w=8$ & 13.55 & 15.32 $[13.57,17.05]$ & 0.00 $[0.00,1.83]$ & 0.00\\
Posterior mode, Dirichlet$(1/2)$ & $\theta=0.95$ & 12.54 & 21.62 $[19.12,24.10]$ & 2.30 $[1.60,3.05]$ & $-0.15$ $[-0.70,+0.40]$\\
\bottomrule
\end{tabular}%
}
\end{table}

\FloatBarrier

\section{Reproducibility Details}
\label{app:checkpoints}

The five checkpoints use official BF16 weights with text-only inputs and SDPA attention: Qwen3.5 \citep{qwen2026qwen35}, Gemma 4 \citep{gemmateam2026gemma4}, and Ministral 3 \citep{liu2026ministral3}. Qwen3.5 disables thinking, Gemma 4 omits its thinking token, and Ministral 3 uses its instruction checkpoint. Unless stated otherwise, in-house generation used Transformers 5.14.1, Accelerate 1.14.0, and \texttt{mistral-common} 1.11.5. \Cref{tab:checkpoint-identities} lists the pinned checkpoint revisions.

\paragraph{Generation settings.}
GSM-Symbolic uses native nonthinking chat templates with temperature $0.7$, top-$p=0.9$, top-$k=20$, repetition penalty $1.05$, and a 1,024-token generation cap. Native MMLU-Pro and Word16 use the same temperature, top-$p$, top-$k$, and repetition penalty, with no presence penalty and a maximum of 4,096 new tokens. The diversification arms inherit the native settings except that High temperature uses temperature $1.0$. The physical stopping runs use the native nonthinking MMLU-Pro protocol on one RTX PRO 6000 Blackwell GPU with BF16, Transformers/PyTorch SDPA, and the torch hybrid-attention fallback.

\begin{table}[htbp]
\centering\scriptsize\setlength{\tabcolsep}{4pt}
\caption{Checkpoint identities and pinned revisions.}
\label{tab:checkpoint-identities}
\begin{tabularx}{\linewidth}{>{\raggedright\arraybackslash}p{0.46\linewidth}X}
\toprule
Checkpoint & Revision\\
\midrule
\path{Qwen/Qwen3.5-0.8B} & \texttt{2fc06364} \texttt{715b967f} \texttt{1860aea9} \texttt{cf387788} \texttt{75588b17}\\
\path{Qwen/Qwen3.5-2B} & \texttt{15852e8c} \texttt{16360a2f} \texttt{ea060d61} \texttt{5a32b452} \texttt{70f8a8fc}\\
\path{Qwen/Qwen3.5-4B} & \texttt{851bf6e8} \texttt{06efd8d0} \texttt{a36b00dd} \texttt{f55e13cc} \texttt{b7b8cd0a}\\
\path{google/gemma-4-E2B-it} & \texttt{3e22461f} \texttt{65e89153} \texttt{144f8adb} \texttt{70e3b8c2} \texttt{cc9845a7}\\
\path{mistralai/Ministral-3-3B-Instruct-2512-BF16} & \texttt{b6d637be} \texttt{f2393152} \texttt{b3da2b2f} \texttt{de72eecd} \texttt{ee30557e}\\
\path{Qwen/Qwen3.5-27B} (judge and forced-choice replay, \cref{app:selector-intervention,app:scope-scale}) & \texttt{fc05daec} \texttt{18b0a78c} \texttt{049392ed} \texttt{2e771dde} \texttt{82bdf654}\\
{\path{Qwen/Qwen3-14B} (thinking mode, \cref{app:thinking-tokens})} & \texttt{40c06982} \texttt{4f4251a9} \texttt{1eefaf28} \texttt{1ebe4c54} \texttt{4efd3e18}\\
{\path{Qwen/Qwen3-8B} (thinking mode, \cref{app:thinking-tokens})} & \texttt{b968826d} \texttt{9c46dd60} \texttt{66d109ea} \texttt{bc625518} \texttt{8de91218}\\
\bottomrule
\end{tabularx}
\end{table}

\paragraph{Forced-choice MMLU-Pro prompt.}
The system instruction is: ``You are one sampled reasoning agent in a controlled multiple-choice study. Solve the problem independently, keep the rationale concise, and follow the requested output schema exactly.'' The user message requests a concise rationale of at most 40 plain-text words and specifies that a separate forced-choice scoring step determines the option ordering. No candidate pool or previous response is appended.

\paragraph{Parsers and invalid responses.}
Forced-choice MMLU-Pro scores the answer options after rationale processing. GSM-Symbolic parses the last boxed exact rational. Native MMLU-Pro parses a canonical option letter, and Word16 parses a whitespace-normalized word sequence. In every protocol, an invalid completion consumes a call and contributes no vote.

\paragraph{Later studies and public data.}
The effect-forecast study and the fresh native cohort use Qwen3.5-4B with the native settings above. The judge study keeps the forced-choice prompt and scoring for direct solving, generates judge rationales with at most 96 new tokens and scores one letter over the full option set. The public GSM8K samples are the Llama-3-8B-Instruct and Llama-3-70B-Instruct completions released by \citet{brown2024largelanguagemonkeys} (dataset revision \texttt{a9f8f73b}). We generated none of them. The public MATH samples come from the same release (Llama-3-8B-Instruct and Llama-3-70B-Instruct, dataset revision \texttt{a9f8f73b}). The Qwen3.5-27B forced-choice sequences were generated for the judge study. The fresh MMLU-Pro temperature study (\cref{app:temp-perm}) uses Qwen3.5-4B with the native settings and reaches the 4,096-token cap by staged continuation. The thinking-mode study (\cref{app:thinking-tokens}) runs Qwen3-14B and Qwen3-8B in vLLM 0.13.0 \citep{kwon2023vllm}.

\FloatBarrier

\end{document}